\documentclass[lettersize,journal]{IEEEtran}

\IEEEoverridecommandlockouts                              

\usepackage{mathtools}
\usepackage{gensymb}
\usepackage[export]{adjustbox}
\usepackage{color}
\usepackage{graphics} 
\usepackage{epsfig} 
\usepackage{times} 
\usepackage{amsmath} 
\usepackage{amssymb}  
\usepackage{leftidx}
\newtheorem{theorem}{Theorem}
\newtheorem{remark}{Remark}

\newtheorem{property}{Property}

\newtheorem{mydef}{Definition}
\usepackage{siunitx}
\usepackage{cite}[noadjust]
\usepackage{nomencl}
\makenomenclature
\usepackage{hyperref}
\DeclareMathOperator{\diag}{diag}

\usepackage{enumitem}
\usepackage{algorithm,algorithmic}
\usepackage{tablefootnote}
\usepackage{ushort}
\newcommand{\norm}[1]{\left\lVert#1\right\rVert}

\usepackage{stackengine}
\usepackage{multirow}  

\title{
A Switched Adaptive Control Framework for Aerial Manipulators Under Dynamic Transitions
}

\author{Rishabh Yadav$^{1}$, Saksham Gupta$^{2}$, Amitabh Sharma$^{2}$, Sarthak~Mishra$^{2}$, Wei Pan$^{3}$, Spandan Roy$^{2}$, Simone Baldi$^{4}$
\thanks{The work is partly supported by the UASAT project sponsored by MeiTY, India, by the National Natural Science Foundation
of China No. 62673118, and by Jiangsu Provincial Scientific Research Center of Applied Mathematics No. BK20233002.}
\thanks{$^{1}$ The author is with Department of Computer Science, The University of Manchester, UK {(email: rishabh.yadav@postgrad.manchester.ac.uk).}   }

\thanks{$^{2}$ The authors are with Robotics Research Center, International Institute of Information Technology Hyderabad, India {(email: \{saksham.g, amitabh.sharma, sarthak.mishra\}@research.iiit.ac.in, spandan.roy@iiit.ac.in).}   }
\thanks{$^{3}$ W. Pan is with the School of Engineering, Newcastle University, UK. (e-mail: wei.pan2@newcastle.ac.uk).
}
\thanks{$^{4}$ S. Baldi is with the School of Mathematics, Southeast University, Nanjing,
China (e-mail: simonebaldi@seu.edu.cn).
}

 \thanks{ 
\textit{Corresponding authors: S. Baldi and S. Roy}}
}

\begin{document}

\maketitle
\thispagestyle{plain}
\pagestyle{plain}

\setlength{\belowcaptionskip}{-10pt}

\begin{abstract}
Aerial manipulators represent the forefront of aerial robotics. Although potentially capable of complex interaction tasks, controlling aerial manipulators throughout the dynamic transitions occurring during task execution presents significant challenges. Abrupt or discontinuous changes in system dynamics generated by the transitions suggest the use of a switched approach, yet the available aerial manipulation methods are not designed for coping with switched regimes. In addition, most available methods fall short in coping with the tight couplings between the aerial vehicle and the manipulator, as well as in coping with the state-dependent uncertainties arising from the difficulty in modeling such couplings. We propose a switched-based adaptive control framework for aerial manipulators not relying on a priori knowledge of the vehicle-manipulator couplings and of state-dependent uncertainties. To guarantee stable manipulation despite changes in system dynamics, the framework provides a class of switching signals characterizing those transition phases for which the system is guaranteed to remain stable. Comparative experiments further validate the effectiveness of the proposed switched-based framework over the state of the art.
\end{abstract}
\begin{IEEEkeywords}
Adaptive control, Aerial manipulator, Dynamic transitions, Switching control.
\end{IEEEkeywords}

\section{Introduction}
Aerial manipulators enhance dexterity by integrating robotic arms with aerial vehicle platforms, enabling operations in scenarios like disaster response, inspection, and construction \cite{ tognon2019truly, 10339889, 10758214}. However, these and other scenarios often involve sudden dynamic transitions during task execution (pick-and-place, contact-based inspection,  etc.) \cite{singh2026aerograb, mishra2026aeroplace}. %

As a representative example of dynamic
transitions, take a payload pick-up and transportation task, one of the most fundamental yet challenging tasks for an aerial manipulator \cite{orsag2017dexterous, ruggiero2018aerial, meng2020survey, ujjawal2026learn, yadav2025integrated}. The fundamental challenge during such task is that the dynamics of the aerial manipulator do not evolve according to a single operational regime, but make transitions between different regimes, such as free-flight without payload, object grasping, flight with payload and payload release. It is impossible for a single model to capture all the regimes, as each regime is governed by different dynamical properties and interaction forces \cite{ujjawal2026aermanidiffusion}. This regime-dependent behavior arises in several other tasks: it results in switched dynamics that are the main reason why it is challenging for non-switched control methods to compensate for variations in inertia distribution \cite{tilli2021low, xiao2021prescribed, 9462539, 10466505, cao2023eso}, as well as compensate for coupling forces between the aerial vehicle and the manipulator as the system moves or interacts with objects \cite[Ch. 5.3]{orsag2018aerial}, \cite{orsag2017dexterous, cao2019faster}. 
Advanced control strategies for aerial manipulators are required to ensure stability and accuracy throughout such dynamic transitions \cite{suarez2020benchmarks, 9813358}. In the following, we discuss the limitations of state-of-the-art controllers developed for aerial manipulators, which lead to the motivation of this work.

\subsection{Related Works and Motivation}
A few strategies based on robust and adaptive control have been developed to specifically address aerial manipulators with modeling uncertainties. 
However, by not considering switched behavior, these strategies exhibit intrinsic limitations. Robust control strategies, using either high-gain observers \cite{kim2017robust}, disturbance observers \cite{chen2020robust}, or RISE-based methods \cite{lee2022rise}, rely on nominal system knowledge and predefined uncertainty bounds, making them inflexible to accommodate switched, state-dependent changes in system dynamics. Adaptive control approaches—including adaptive disturbance observers \cite{liang2021low, 10466505}, adaptive backstepping \cite{10505853}, adaptive sliding mode observers \cite{chen2022adaptive}, and adaptive sliding mode control \cite{kim2016vision, liang2022adaptive}—assume smoothly evolving uncertainties and require prior boundedness of the uncertainties and of their time derivative, making them unsuitable for switched state-dependent variations. Even when the requirement of a priori boundedness is relaxed (cf. \cite{liang2022adaptive}), system stability still relies on the knowledge of the inertial couplings 
between the aerial vehicle and the manipulator. The term ``inertial couplings'' refers to couplings between the subsystems of the aerial manipulator (vehicle position, vehicle attitude, and manipulator), acting via off-diagonal terms in the mass/inertia matrix. 
These coupling terms are difficult to model as they strongly depend on the interactions between the manipulator and the environment \cite[Ch. 5.3]{orsag2018aerial}. 
To compensate vehicle-manipulator inertial couplings, \cite{10722859} requires a disturbance model. Although \cite{10701509} treated these couplings as state-dependent uncertainties, both \cite{10722859} and \cite{10701509} assume continuous evolution of the system dynamics, failing to account for regime-dependent transitions. It is also worth mentioning that the design of common adaptive gains shared among all the subsystems of the aerial manipulator strongly limits adaptability during transitions (cf. \cite{liang2021low, chen2022adaptive, kim2016vision, liang2022adaptive}).

One may wonder whether existing adaptive controllers for nonlinear switched systems \cite{lai2018adaptive, yuan2018robust, yuan2016adaptive, lou2018immersion, 8249868, 8362915, roy2019reduced, roy2019simultaneous} are applicable to switched aerial manipulator dynamics. Unfortunately, the answer is negative, as the existing switched adaptive controllers either rely on a priori known structure of the system dynamics (cf. \cite{lai2018adaptive, yuan2018robust, yuan2016adaptive, lou2018immersion, 8249868, 8362915}) or require known bounds on the inertia matrix (cf. \cite{roy2019reduced, roy2019simultaneous}). 

\subsection{Contributions of This Study}
The above discussion reveals that a control method for aerial manipulators capable of handling state-dependent uncertainties and inertial couplings under regime-dependent dynamic variations is a significant open challenge. This study focuses on the aforementioned representative example of payload pick-up and transportation, although the switched-based framework we propose is potentially applicable to other tasks. The challenges are solved with the following contributions:

\begin{itemize}
    \item Modeling the regime-dependent behavior of an aerial manipulator using a switched Euler-Lagrange framework, marking a difference with methods \cite{liang2021low, 10466505, 10505853, chen2022adaptive, kim2016vision, liang2022adaptive, 10722859, 10701509} that assume continuously evolving dynamics.
    
    \item  Designing a switched controller that accounts for the switched dynamics of the aerial manipulator without a priori knowledge of neither state-dependent uncertainties (as opposed to relying on nominal dynamics or uncertainty bounds \cite{kim2017robust, zhang2019robust, chen2020robust,lee2022rise, liang2021low, chen2022adaptive, kim2016vision }) nor vehicle-manipulator inertial couplings (unlike \cite{jimenez2013control, ruggiero2015multilayer, liang2022adaptive}). 
    
   \item To guarantee stable manipulation despite changes in system dynamics, we provide not only a stabilizing control law, but also a class of stabilizing switching signals 
   used to characterize all those transition phases for which the system is guaranteed to remain stable.
   
   \item As compared to existing switched adaptive controllers \cite{lai2018adaptive, yuan2018robust, yuan2016adaptive, lou2018immersion, 8249868, 8362915, roy2019simultaneous}, we do not require structural knowledge of the system terms, which creates a key challenge in orchestrating the adaptation during active and inactive regimes. Such an orchestration challenge is solved via a properly designed auxiliary gain.
\end{itemize}

The proposed switched-based framework is validated through comparative real experiments and {\color{black}simulations}, demonstrating superior tracking accuracy and robustness.

\section{System Dynamics and Problem Formulation}
The following notations are used in this paper:  $\norm {(\cdot)}$ and $\lambda_{\min}(\cdot)$ denote the 2-norm and minimum eigenvalue of $(\cdot)$, respectively; $\boldsymbol I$ denotes the identity matrix with appropriate dimension and $\diag\lbrace \cdot, \cdots, \cdot \rbrace$ denotes a diagonal matrix.

\subsection{(Non-Switched) Aerial Manipulator Dynamics}
\begin{figure}[]
\begin{center}
    \includegraphics[scale=0.03]{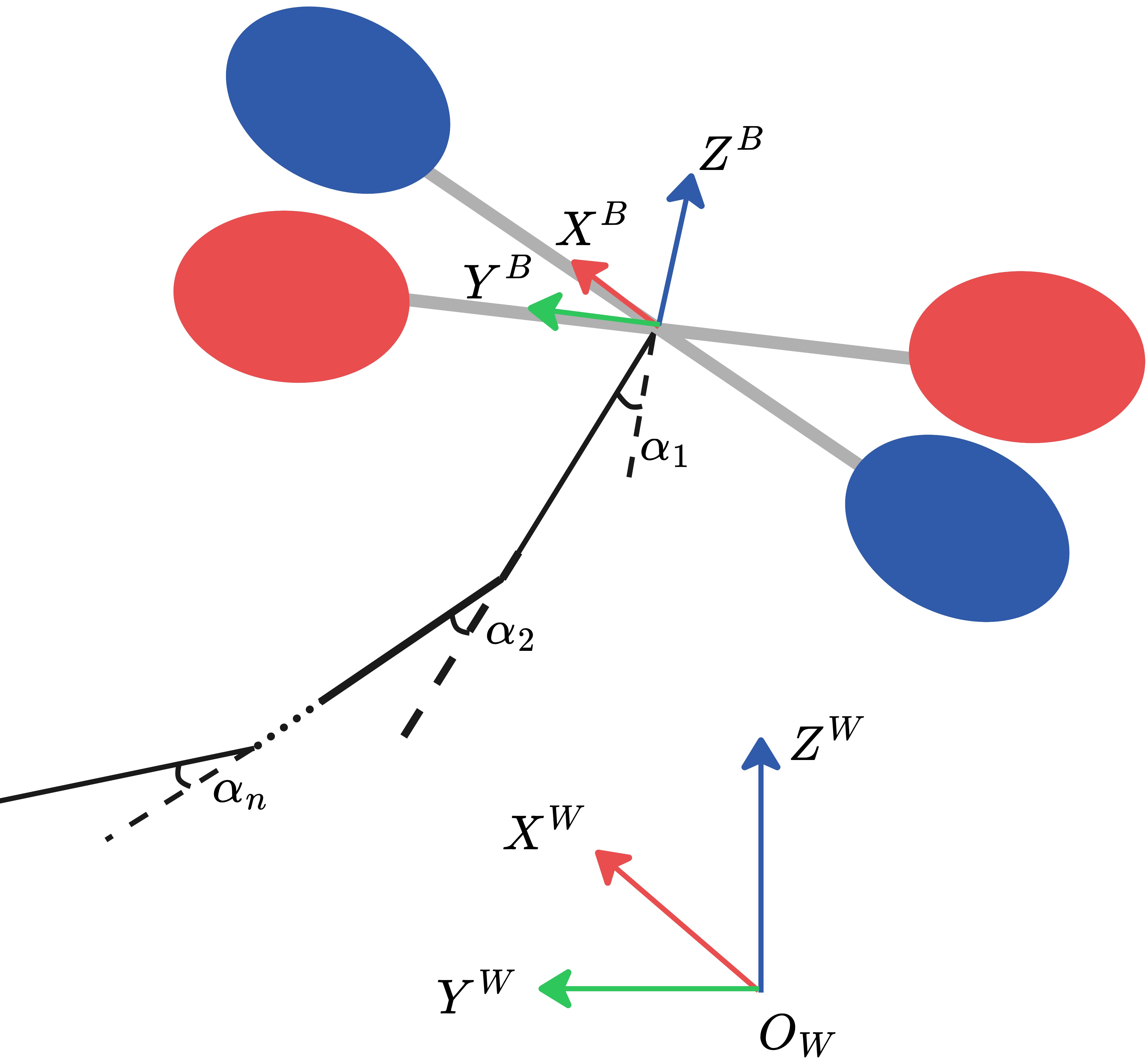}
    \caption{Schematic of a quadrotor-based aerial manipulator with an $n$-link manipulator and the corresponding frames.}
    \label{robot_model}
\end{center}
\end{figure}

\begin{table}[!t]
\renewcommand{\arraystretch}{0.99}
\caption{{Nomenclature}}
\label{table_nomenclature}
\centering
{
{	\begin{tabular}{c||c}
\hline
\hline~
	$[ X^B ~ Y^B ~Z^B ]$& Quadrotor body-fixed coordinate frame\\
	$[ X^W ~ Y^W ~ Z^W]$ & Earth-fixed coordinate frame \\
    $\boldsymbol{R}_B^W \in \mathbb{R}^{3 \times 3}$ & $Z$-$Y$-$X$ Euler angle rotation matrix \\
	$p=[x ~ y ~z]^{\top}$ & Quadrotor position in $[ X^W ~ Y^W ~ Z^W]$\\
	$q=[\phi ~ \theta ~ \psi]^{\top}$ & Quadrotor roll, pitch and yaw angles\\
	$\alpha=[\alpha_1, ~ \alpha_2, ~\cdot \cdot ~, \alpha_n]^{\top}$ & Manipulator joint angles\\
	$\boldsymbol{M, C}\in \mathbb{R}^{(6+n)\times (6+n)}$ & Mass and Coriolis matrix \\ 
	$g\in \mathbb{R}^{6+n}$, $d\in \mathbb{R}^{6+n} $ &  Gravity and disturbance terms \\
	$\tau_p, \tau_q \in \mathbb{R}^3$ & Generalized quadrotor control inputs \\
 $\tau_{\alpha} \triangleq
	\begin{bmatrix}
	\tau_{\alpha_1},~ \tau_{\alpha_2},~\cdots, \tau_{\alpha_n}
	\end{bmatrix}^{\top}$ & Manipulator's joint control inputs \\
\hline 
\hline
\end{tabular}}}
\end{table}
The conventional Euler-Lagrange (EL) dynamical model of a quadrotor vehicle equipped with an $n$ degrees-of-freedom (DoF) manipulator (cf. Fig. \ref{robot_model} and Table \ref{table_nomenclature} for the meaning of  symbols and system parameters) is given by \cite{arleo2013control}
\begin{equation} \small \label{EL_dynamics}
\boldsymbol M(\chi(t))\ddot{\chi}(t) +\boldsymbol C(\chi(t), \dot{\chi}(t))\dot{\chi}(t) + g(\chi(t)) + d(t) = \tau(t),    
\end{equation}
where $\chi = \begin{bmatrix} p^{\top} & q^{\top} & \alpha^{\top}\end{bmatrix}^{\top}, \tau = \begin{bmatrix}
    \tau_p^{\top} &    \tau_q^{\top} &    \tau_\alpha^{\top} \end{bmatrix}^{\top}\in \mathbb{R}^{6+n}$ denote generalized coordinates and generalized inputs, respectively. 
The terms in (\ref{EL_dynamics}) can be decomposed along the quadrotor position, quadrotor attitude, and manipulator subsystems  as
\begin{subequations}\label{split_2}
\begin{align}
&\boldsymbol M = \begin{bmatrix}
    \boldsymbol M_{pp} & \boldsymbol M_{pq} & \boldsymbol M_{p\alpha} \\
    \boldsymbol M_{pq}^{\top} & \boldsymbol M_{qq} & \boldsymbol M_{q\alpha} \\
    \boldsymbol M_{p\alpha}^{\top} & \boldsymbol M_{q\alpha}^{\top} & \boldsymbol M_{\alpha \alpha}
    \end{bmatrix},~\begin{matrix}
\boldsymbol M_{pp}, \boldsymbol M_{qq}, \boldsymbol M_{pq} \in \mathbb{R}^{3\times3}\\ 
\boldsymbol M_{p\alpha}, \boldsymbol M_{q\alpha} \in \mathbb{R}^{3 \times n}\\ 
\boldsymbol M_{\alpha \alpha} \in \mathbb{R}^{n \times n}
\end{matrix} \label{mass_split}\\
&\boldsymbol C = \begin{bmatrix}
    \boldsymbol C_p \\
     \boldsymbol C_q \\
    \boldsymbol C_\alpha
    \end{bmatrix}, ~\begin{matrix} \boldsymbol C_p , \boldsymbol C_q \in \mathbb{R}^{3 \times (6+n)} \\ \boldsymbol C_\alpha \in \mathbb{R}^{n \times (6+n)}  \end{matrix} \label{new_dyn1}\\
&g = \begin{bmatrix}
    g_p \\
    g_q \\
    g_\alpha
    \end{bmatrix},
    d = \begin{bmatrix}
    d_p \\
    d_q \\
    d_\alpha
    \end{bmatrix}, ~\begin{matrix} g_p, g_q,d_p,d_q \in \mathbb{R}^{3}\\ g_\alpha, d_\alpha \in \mathbb{R}^{n}\end{matrix}. \label{new_dyn2}
\end{align}
\end{subequations}
 Here $\tau_q  \triangleq 
	\begin{bmatrix}
	u_2(t) & u_3(t) & u_4(t)
	\end{bmatrix}^{\top}$ is the roll, pitch and yaw control inputs for the quadrotor; ${\tau_{p}} = {\boldsymbol R^W_B U}$ is the generalized control input for quadrotor position in Earth-fixed frame, being ${U}(t)\triangleq
	\begin{bmatrix}
	0 & 0 & u_1(t)
	\end{bmatrix}^{\top}\in \mathbb{R}^3$ the force vector in body-fixed frame, $u_1$ the total thrust and ${\boldsymbol R^W_B} \in\mathbb{R}^{3\times3}$ the $Z$-$Y$-$X$ orthogonal Euler angle rotation matrix describing the rotation from the body-fixed to the Earth-fixed frame \cite{mellinger2011minimum}. We stress that the off-diagonal terms of $\mathbf M$ in (\ref{mass_split}), represent the vehicle-manipulator couplings, that is, the inertial couplings between the quadrotor position, quadrotor attitude, and manipulator subsystems. These terms, along with the coupling terms appearing in $\boldsymbol{C}$, are known to be difficult, if at all possible, to model \cite[Ch. 5.3]{orsag2018aerial}.

\begin{figure}[!b]
\centering
\includegraphics[width=0.48\textwidth]{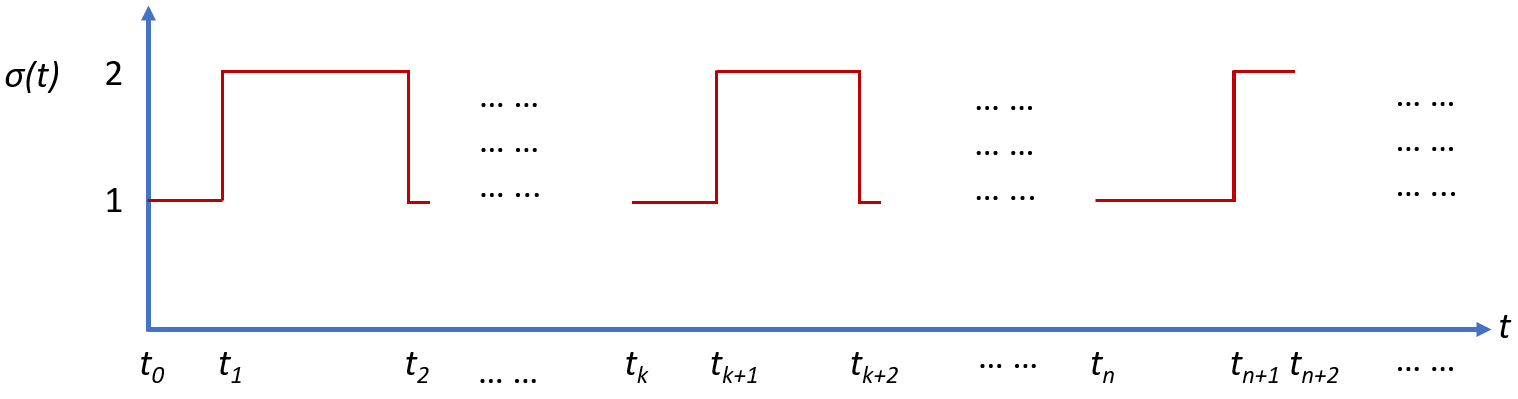}
 \caption{A representative switching signal $\sigma$. }
\label{fig:ref_sig}
\end{figure}

\subsection{Switching Signal and Preliminaries}
The changes in the aerial manipulator dynamics during payload pick-and-place are not restricted to differences in the system mass, but also result from the different configurations of the manipulator during transportation and release, which lead to varying inertia distributions. To capture different operational regimes, this work considers a piecewise constant function of time $\sigma(\cdot): [0~\infty) \mapsto \Omega $ taking values in $\Omega = \lbrace 1, 2 \rbrace $,  
where (i) $\sigma(t)=2$ when the aerial manipulator picks up and carries a payload; (ii) $\sigma(\cdot)=1$ when the aerial manipulator releases the payload. In switched systems literature, $\sigma(\cdot)$ is known as the \emph{switching signal}  \cite{liberzon2003switching, hespanha1999stability}, and can be extended to deal with more regimes than the ones considered in this work. An illustrative example of $\sigma(t)$ is shown in Fig. \ref{fig:ref_sig}, where $t_k$ denotes the switching instants (i.e. {when} the operational regime of the quadrotor changes) and the time $t \in [t_k~~t_{k+1})$ between consecutive switching instants is called the `\emph{dwell-time}'.  
We follow an average dwell-time description of the switching signal: we first give its formal definition and then explain its rationale.

\begin{mydef} \label{def1}
(Average Dwell-Time (ADT) \cite{liberzon2003switching}): For a switching signal $\sigma(\cdot)$ and time instants $T_2 \geq T_1 \geq 0$, let $N_{\sigma}(T_1,T_2)$ denote the number of switchings in the interval $[T_1,T_2)$. Then, $\sigma(\cdot)$ has an average dwell time $\vartheta$ if
\begin{align}
N_{\sigma}(T_1,T_2) &\leq N_0 + (T_2-T_1)/\vartheta,~~ \forall T_2 \geq T_1 \geq 0, \label{ADT}
\end{align} 
\end{mydef}
where $N_0 >0$ is a scalar termed as chatter bound. Since $N_0>0$, the ADT in (\ref{ADT}) implies
\begin{align}
    \vartheta \leq \frac{T_2-T_1}{N_\sigma(\cdot)-N_0}. \label{ADT_new}
\end{align}

\begin{remark}[ADT rationale]\label{rem_adt}
In practice, 
the length of the switching intervals $t_{k+1}-t_k$
may vary during operation (cf. Fig. \ref{fig:ref_sig}). {\color{black}The notion of ADT in (\ref{ADT_new}) captures these variations by defining the duration of the interval between switching instants in an average sense: this allows to have short dwell-times $t_{k+1}-t_k < \vartheta $ provided that they are compensated by longer dwell-times, i.e., $t_{k+1}-t_k>\vartheta$.} According to (\ref{ADT_new}), the chatter bound gives additional flexibility in distinguishing between shorter and longer time intervals.
\end{remark}

\subsection{Switched Model of an Aerial Manipulator}
Based on the discussions on the previous subsection and the EL dynamics (\ref{EL_dynamics}), the switched dynamics of an aerial manipulator can be formulated as
\begin{align} 
\small \label{switched_EL_dynamics}
\boldsymbol M_{\sigma(t)}(\chi(t))\ddot{\chi}(t) + &\boldsymbol C_{\sigma(t)}(\chi(t), \dot{\chi}(t))\dot{\chi}(t) +  g_{\sigma(t)}(\chi(t)) \nonumber\\
&+ d_{\sigma(t)}(t) = \tau_{\sigma(t)}.   
\end{align}
{\color{black}Note that a switching disturbance $d_\sigma$ may represent bounded contact forces that appear during payload grasping ($\sigma=2$) and disappear during release ($\sigma=1$).}
In the following, let us omit time and variable dependence whenever obvious. For the ease of control design and analysis and using (\ref{split_2}), the switched dynamics (\ref{switched_EL_dynamics}) can be rewritten as
\begin{subequations} 
\label{dynamics}
\small
\begin{align}
&\boldsymbol M_{pp\sigma}\ddot{p} + \boldsymbol M_{pq\sigma}\ddot{q} + \boldsymbol M_{p\alpha\sigma}\ddot{\alpha} + \boldsymbol C_{p\sigma}\dot{\chi} + g_{p\sigma} + d_{p\sigma} = \tau_{p\sigma},  \label{pos}\\
&\boldsymbol M_{pq\sigma}^{\top}\ddot{p} + \boldsymbol M_{qq\sigma}\ddot{q} + \boldsymbol M_{q\alpha\sigma}\ddot{\alpha} + \boldsymbol C_{q\sigma}\dot{\chi} + g_{q\sigma} + d_{q\sigma} = \tau_{q\sigma}, \label{att}\\
&\boldsymbol M_{p\alpha\sigma}^{\top}\ddot{p} + \boldsymbol M_{q\alpha\sigma}^{\top}\ddot{q} + \boldsymbol M_{\alpha\alpha\sigma}\ddot{\alpha} + \boldsymbol C_{\alpha\sigma} \dot{\chi} + g_{\alpha\sigma} + d_{\alpha\sigma} \hspace{-0.5mm}= \hspace{-0.5mm}\tau_{\alpha\sigma}, \label{man}
\end{align}
\end{subequations}
where (\ref{pos}), (\ref{att}) and (\ref{man}) represent the quadrotor position subsystem, quadrotor attitude subsystem and manipulator subsystem along with their couplings, respectively.

The following standard system properties hold from the EL mechanics \cite{spong2008robot}:
\begin{property} \label{prop_1}
{$\boldsymbol M_\sigma(\chi)$} is uniformly positive definite $\forall \chi$ and $ \exists \ushort{m}_\sigma,\overline{m}_\sigma \in \mathbb{R}^{+}$ such that $0 < \ushort{m}_\sigma \boldsymbol I \leq \boldsymbol M_\sigma(\chi) \leq \overline{m}_\sigma \boldsymbol I$. 
\end{property}

\begin{property} \label{prop_2}
$\exists \bar{c}, \bar{g}, \bar{d}\in\mathbb{R}^{+}$ such that $||\boldsymbol C_\sigma (\chi, \dot{\chi})|| \leq \bar{c}_ \sigma||\dot{\chi}||$, $||g_\sigma (\chi)|| \leq \bar{g}_\sigma$ and $||d_\sigma(t)|| \leq \bar{d}_\sigma$. This implies, from (\ref{new_dyn1})-(\ref{new_dyn2}), $\exists \bar{c}_{p \sigma}, \bar{c}_{q\sigma}, \bar{c}_{\alpha\sigma}, \bar{g}_{p\sigma}, \bar{g}_{q\sigma}, \bar{g}_{\alpha\sigma}, \bar{d}_{p\sigma}, \bar{d}_{q\sigma}, \bar{d}_{\alpha\sigma} \in\mathbb{R}^{+}$ such that: $||\boldsymbol C_{p\sigma} (\chi, \dot{\chi})|| \leq \bar{c}_{p\sigma}||\dot{\chi}||, ||\boldsymbol C_{q\sigma} (\chi, \dot{\chi})|| \leq \bar{c}_{q\sigma}||\dot{\chi}||, ||\boldsymbol C_{\alpha\sigma} (\chi, \dot{\chi})|| \leq \bar{c}_{\alpha\sigma}||\dot{\chi}||$, $||g_{p\sigma} (\chi)|| \leq \bar{g}_{p\sigma}, ||g_{q\sigma} (\chi)|| \leq \bar{g}_{q\sigma}, ||g_{\alpha\sigma} (\chi)|| \leq \bar{g}_{\alpha \sigma}$, $||d_{p\sigma} (t)|| \leq \bar{d}_{p\sigma}, ||d_{q \sigma} (t)|| \leq \bar{d}_{q \sigma}, ||d_{\alpha \sigma} (t)|| \leq \bar{d}_{\alpha \sigma}$.
\end{property}

In the following, we highlight the model uncertainties considered in this work.
\begin{remark}[Uncertainty] \label{remark_uncertainity}
All the system dynamics terms $ \boldsymbol M_\sigma, \boldsymbol C_{p \sigma}, \boldsymbol C_{q\sigma}, \boldsymbol C_{\alpha\sigma}, g_{p\sigma}, g_{q\sigma}, g_{\alpha\sigma}, d_{p\sigma}, d_{q\sigma}, d_{\alpha\sigma}$
and their bounds $\overline{m}_\sigma, \ushort{m}_\sigma, \bar{c}_{p\sigma}, \bar{c}_{q\sigma}, \bar{c}_{\alpha\sigma}, \bar{g}_{p\sigma}, \bar{g}_{q\sigma}, \bar{g}_{\alpha\sigma}, \bar{d}_{p\sigma}, \bar{d}_{q\sigma}, \bar{d}_{\alpha\sigma}$ defined in Properties 1-2 are unknown for control design. {\color{black}Note that state jump discontinuities caused by momentary impulses during payload grasp and release are
assumed negligible compared to the continuous motion dynamics.}
\end{remark}

As common in the literature, the desired trajectories $\chi_d = \begin{bmatrix} p_d^{\top} & q_d^{\top} & \alpha_d^{\top}\end{bmatrix}^{\top}$ and their time-derivatives $\dot{\chi}_d, \ddot{\chi}_d$ are designed to
be bounded. Furthermore, $\chi, \dot{\chi}, \ddot{\chi}$ are considered to be available for feedback. Let us define the tracking error as 
\begin{align}
  e(t) \triangleq \chi(t) - \chi_d(t),  ~\xi(t) \triangleq \begin{bmatrix}
	e^{\top}(t) & \dot{e}^{\top}(t)
	\end{bmatrix}^{\top}. \label{err}
\end{align}
We recall the fact that stabilization of each single regime (i.e., with a fixed $\sigma$) is not enough to guarantee stability of the switched dynamics, i.e., with a time-varying $\sigma(\cdot)$), see \cite{lai2018adaptive, yuan2018robust, yuan2016adaptive, lou2018immersion, 8249868, 8362915, roy2019reduced, roy2019simultaneous}. As a result,  
the switching law $\sigma(\cdot)$ contributes to determine the stability of the aerial manipulator system. Accordingly, the control problem is formulated as follows.
\textit{Control Problem:} Under Properties 1-2 and subject to  the state-dependent uncertainties described in Remark 2, the control problem is to design a switched adaptive control and a stabilizing class of ADT switching signals  for the aerial manipulator switched dynamics (\ref{switched_EL_dynamics}). 

\section{Proposed Switched Adaptive Control Design}
To present a unified formulation applicable to all subsystems, let \( j \in \{p, q, \alpha\} \) denote the subsystem index corresponding to the \emph{quadrotor position}, \emph{attitude}, and \emph{manipulator} dynamics, respectively. The control law and adaptive update rules are then derived in compact form for each subsystem, followed by the definition of the ADT-based switching condition that guarantees closed-loop stability.

\subsection{Position, Attitude and Manipulator sub-systems control}
Taking the tracking error as $e_j \triangleq j - j_d$, let us define an error variable as
\begin{align}
     s_j &= \dot{e}_j + \boldsymbol \Phi_j e_j \label{eq:s_j},
 \end{align}
where $\boldsymbol \Phi_{j} $ is a positive definite matrix. Multiplying the time derivative of (\ref{eq:s_j}) by a user-defined constant positive definite matrix $\bar{\boldsymbol{M}}_{jj\sigma}  $ and using (\ref{dynamics}) yields
\begin{align}
\bar{\boldsymbol{M}}_{jj\sigma} \dot{s}_j &= \bar{\boldsymbol{M}}_{jj\sigma}(\ddot{j}- \ddot{j}_d+ \boldsymbol \Phi_{j} {\dot{e}_j}) \nonumber \\ 
&= {\tau_{j\sigma}} + \mathcal{E}_{j\sigma} - \bar{\boldsymbol{M}}_{jj\sigma}(\ddot{j}_d- \boldsymbol \Phi_{j} {\dot{e}_j}), \label{eq:sj_dot}
\end{align}
where $\mathcal{E}_{j\sigma}$ denotes the \textit{lumped uncertainty} associated with each subsystem. 
The term aggregates all unmodeled dynamics, inter-subsystem coupling effects, inertia mismatches, and external disturbances.
Because the coupling structure and parameter dependencies differ across subsystems, $\mathcal{E}_{j\sigma}$ is defined separately as

\begin{subequations} 
\footnotesize
\begin{align*}
&\mathcal{E}_{p\sigma} \triangleq -(\boldsymbol M_{pp\sigma} - \bar{\boldsymbol M}_{pp\sigma})\ddot{p} - \boldsymbol M_{pq\sigma}\ddot{q} - \boldsymbol M_{p\alpha\sigma}\ddot{\alpha} - \boldsymbol C_{p\sigma}\dot{\chi} - g_{p\sigma} - d_{p\sigma} \nonumber \\
&\mathcal{E}_{q\sigma} \triangleq -({\boldsymbol M}_{qq\sigma} - \bar{\boldsymbol M}_{qq\sigma})\ddot{q} - {\boldsymbol M}_{pq\sigma}^{\top}\ddot{p} - {\boldsymbol M}_{q\alpha\sigma}\ddot{\alpha} - {\boldsymbol C}_{q\sigma}\dot{\chi} - g_{q\sigma} - d_{q\sigma} \nonumber \\
& \mathcal{E}_{\alpha\sigma} \hspace{-0.5mm} \triangleq \hspace{-0.5mm}-({\boldsymbol M}_{\alpha\alpha\sigma} - \bar{\boldsymbol M}_{\alpha\alpha\sigma})\ddot{\alpha} - {\boldsymbol M}_{p\alpha\sigma}^{\top}\ddot{p} - {\boldsymbol M}_{q\alpha\sigma}^{\top}\ddot{q} - {\boldsymbol C}_{\alpha\sigma}\dot{\chi} \hspace{-0.5mm} - \hspace{-0.5mm}  g_{\alpha\sigma} \hspace{-0.5mm} - \hspace{-0.5mm} d_{\alpha\sigma} \nonumber 
\end{align*}
\end{subequations}

The dynamics \eqref{eq:sj_dot} will be used for stability analysis in Sect. \ref{sec:app}, where the matrix  $\bar{\boldsymbol M}_{jj\sigma}$ is instrumental in determining the ADT-based class of stabilizing switching signals (cf. discussions in Sect. \ref{sect_sw}). 
The switching control law for each subsystem is proposed as follows
\begin{subequations}\label{ct}
\begin{align}
\tau_{j\sigma} &= -\boldsymbol \Lambda_{j\sigma} s_j - \Delta \tau_{j\sigma} + \bar{\boldsymbol M}_{jj\sigma}(\ddot{j}_d- \boldsymbol \Phi_{j} {\dot{e}_j}),  \label{tau_j}\\
\Delta \tau_{j\sigma} &= \begin{cases}
    \rho_{j\sigma} \frac{s_{j}}{\norm{s_{j}}}       & ~ \text{if } \norm{s_{j}} \geq \varpi_j,\\
    \label{del_j}
    \rho_{j\sigma} \frac{s_{j}}{\varpi_j}       & ~ \text{if } \norm{ s_{j}} < \varpi_j, \\
    \end{cases} 
\end{align}
\end{subequations}
where $\boldsymbol \Lambda_{j\sigma} $ is a user-defined positive definite matrix, $\varpi_{j} > 0$ is used to avoid chattering, and $\rho_{j\sigma}$ is an adaptive gain to tackle the system uncertainties, whose design will be discussed later.

Using Property~\ref{prop_2} and the inequalities 
$\norm{\ddot{\boldsymbol{\chi}}} \geq \norm{\ddot{j}}$, 
$\norm{\boldsymbol{\xi}} \geq \norm{\dot{\boldsymbol{e}}}$, 
$\norm{\boldsymbol{\xi}} \geq \norm{{\boldsymbol{e}}}$, 
and substituting $\dot{\boldsymbol{\chi}} = \dot{\boldsymbol{e}} + \dot{\boldsymbol{\chi}}_d$, 
one can derive an upper bound for the lumped uncertainty term $\mathcal{E}_{j\sigma}$ as
\begin{align} 
\norm{\mathcal{E}_{j\sigma}} &\leq K_{j0\sigma}^*  +K_{j1\sigma}^*\norm{\xi}+ K_{j2\sigma}^*\norm{\xi}^2 + K_{j3\sigma}^*\norm{\ddot{\chi}},\label{up_boundd_j} 
\end{align}
where the coefficients $K_{ji\sigma}^*$ $(i = 0, \dots, 3)$ are unknown but finite positive scalars, 
given as follows.

\paragraph*{Position subsystem ($j = p$)}
\begin{align*}
K_{p0\sigma}^* &= \bar{g}_{p\sigma} + \bar{d}_{p\sigma} + \bar{c}_{p\sigma}\norm{\dot{\chi}_d}^2,\\
K_{p1\sigma}^* &= 2\bar{c}_{p\sigma}\norm{\dot{\chi}_d},~
K_{p2\sigma}^* = \bar{c}_{p\sigma}, \\
K_{p3\sigma}^* &=  \norm{\boldsymbol M_{pp\sigma} - \bar{\boldsymbol{M}}_{pp\sigma} }+ \norm{\boldsymbol M_{pq\sigma}} + \norm{\boldsymbol M_{p\alpha\sigma}},
\end{align*}

\paragraph*{Attitude subsystem ($j = q$)}
\begin{align*}
K_{q0\sigma}^* &= \bar{g}_{q\sigma} + \bar{d}_{q\sigma} + \bar{c}_{q\sigma}\norm{\dot{\chi}_d}^2,\\
K_{q1\sigma}^* &= 2\bar{c}_{q\sigma}\norm{\dot{\chi}_d},~
K_{q2\sigma}^* = \bar{c}_{q\sigma}, \\
K_{q3\sigma}^* &=  \norm{\boldsymbol M_{qq\sigma} - \bar{\boldsymbol{M}}_{qq\sigma} }+ \norm{\boldsymbol M _{pq\sigma}}+ \norm{\boldsymbol M_{q\alpha\sigma}},
\end{align*}

\paragraph*{Manipulator subsystem ($j = \alpha$)}
\begin{align*}
K_{\alpha0\sigma}^* &= \bar{g}_{\alpha\sigma} + \bar{d}_{\alpha\sigma} + \bar{c}_{\alpha\sigma}\norm{\dot{\chi}_d}^2,\\
K_{\alpha1\sigma}^* &= 2\bar{c}_{\alpha\sigma}\norm{\dot{\chi}_d},~
K_{\alpha2\sigma}^* = \bar{c}_{\alpha\sigma}, \\
K_{\alpha3\sigma}^* &=  \norm{\boldsymbol M_{\alpha\alpha\sigma} - \bar{\boldsymbol{M}}_{\alpha\alpha\sigma} }+ \norm{\boldsymbol M_{p\alpha\sigma}}+ \norm{\boldsymbol M _{q\alpha\sigma}},
\end{align*}

Based on the upper bound structure in (\ref{up_boundd_j}), the gain $\rho_{j\sigma}$ in (\ref{del_j}) is designed as
\begin{equation}
 \small
\rho_{j\sigma} = \hat{K}_{j0\sigma} + \hat{K}_{j1\sigma}\norm{\xi} + \hat{K}_{j2\sigma}\norm{\xi}^2 + \hat{K}_{j3\sigma}\norm{\ddot{\chi}} + \zeta_{j\sigma} + \gamma_{j\sigma}, \label{rho_j} 
\end{equation}
where $\hat{K}_{ji\sigma}$ are the estimates of $K_{ji\sigma}^*$ $i=0,1,2,3$, and $\zeta_{j\sigma}$, $\gamma_{j\sigma}$ are auxiliary gains used for closed-loop stabilization (cf. Remark \ref{remark_zeta_gamma}). 
It can be observed from \eqref{ct} that each $\sigma(t)$ has its own set of fixed and adaptive gains. Let $\sigma$ denote the index for the active regime at time $t$ and $\bar{\sigma}(t)$ denote the index for the inactive regime at time $t$: for example, when $\sigma(t)=1$, we have $\bar{\sigma}(t)=2$ and vice versa. 
Accordingly, the gains are adapted via the following laws:
\begin{subequations}\label{adaptive_law_j}
\begin{align}
&\text{\textbf{Active regime:}}\nonumber\\
&\dot{\hat{K}}_{ji\sigma} = \norm{s_{j}}\norm{\xi}^i - \nu_{ji\sigma} \hat{K}_{ji\sigma},\quad ~i = 0,1,2 \label{adaptive_law_j1}\\
&\dot{\hat{K}}_{j3\sigma} = \norm{s_{j}}\norm{\ddot{\chi}} - \nu_{j3\sigma} \hat{K}_{j3\sigma},\label{adaptive_law_j2}\\
&\dot{\zeta}_{j\sigma} = \begin{cases}
   0, ~~~~~~~~~~~~~~~~~~~~ \text{if } \norm{s_j} \geq \varpi_j,  \\
     -(  1 + (\hat{K}_{j3\sigma}\norm{\ddot{\chi}}+ \sum_{i=0}^{2}\hat{K}_{ji\sigma}\norm{\xi}^i )\norm{s_j} ) \zeta_{j\sigma}  \\
  ~~~ + ~~ \delta_{j\sigma}  ,~~~~~~~~~ \text{if } \norm{s_j} < \varpi_j,
\end{cases} \nonumber \\
&~\dot{\gamma}_{j\sigma} = 0, ~\hat{K}_{ji\sigma}(0) > 0, ~\zeta_{j\sigma} (0)  > 0,  \gamma_{j\sigma} (0) >0,\\
&\text{\textbf{Inactive regime:}}\nonumber\\
&\dot{\hat{K}}_{ji\bar{\sigma}} = ~\dot{\hat{K}}_{j3\bar{\sigma}} =  \dot{\zeta}_{j\bar{\sigma}} = 0,\quad ~i = 0,1,2\label{adaptive_law_j3}\\
&  ~\dot{\gamma}_{j\bar{\sigma}} = -\left(1+ \frac{\varrho_{j\sigma}}{2}\sum \limits_{i=0}^{3} {\hat{K}_{ji\bar{\sigma}}}^2 \right) \gamma_{j\bar{\sigma}} +\epsilon_{j\bar{\sigma}}, \label{adaptive_law_j4}\\ 
&\hat{K}_{ji\bar{\sigma}}(0) > 0, {\zeta}_{j \bar{\sigma}}(0) > 0, {\gamma}_{j \bar{\sigma}}(0) > 0, \label{init}
\end{align}
\end{subequations}
where $\nu_{ji\sigma},\delta_{j\sigma},\epsilon_{j\bar{\sigma}}
\in\mathbb{R}^{+}$ are user-defined scalars and
{\color{black}
$\displaystyle
\varrho_{j\sigma}\triangleq
\frac{
\min\left\{
\lambda_{\min}(\boldsymbol\Lambda_{j\sigma}),
\frac{1}{2}\min_{i\in\{0,1,2,3\}}\nu_{ji\sigma}
\right\}}
{
\max\left\{
\lambda_{\max}(\bar{\boldsymbol M}_{jj\sigma}),1/2
\right\}
}$.
}

\subsection{ADT-based class of stabilizing switching signals} \label{sect_sw}

Let us define the following positive design scalars
\begin{subequations}\label{ADT_gains}\small
\begin{align}
& {\bar{\varrho}_j} \triangleq \max_{\sigma \in \Omega} \{\lambda_{\max} (\bar{\boldsymbol M}_{jj\sigma})\}  ,~ {\ushort{\varrho}_{j}} \triangleq \min_{\sigma \in \Omega} \{\lambda_{\min} (\bar{\boldsymbol M}_{jj\sigma})\} , \\
& \mu_{j}  \triangleq {\bar{\varrho}_{j}} / \ushort{\varrho}_{j}, ~~~ \mu  = \max(\mu_{p},\mu_{q}, \mu_{\alpha}), \label{mu} \\
&\text{and}~ 0 < \kappa < \varrho,~\text{with}~
\varrho =
\min_{j\in\{p,q,\alpha\},\,\sigma\in\Omega}
\left\{\varrho_{j\sigma}\right\}.
\label{mu_def}
\end{align}
\end{subequations}

Based on the above parameters, consider the ADT threshold
\begin{equation}
\vartheta
\triangleq
\frac{\ln\mu}{\kappa} \leq  \frac{T_2-T_1}{N_\sigma(\cdot)} < \frac{T_2-T_1}{N_\sigma(\cdot)-N_0}. 
\label{eq:sw_law}
\end{equation}



\begin{theorem}
Under Properties \ref{prop_1}-\ref{prop_2}, the control laws~(\ref{ct}) and adaptive
laws~(\ref{adaptive_law_j}) make the closed-loop trajectories in~(\ref{eq:s_j}) Uniformly Ultimately
Bounded (UUB) for any switching signal that satisfies the average dwell time condition
in (\ref{eq:sw_law}).
\end{theorem}

\textit{Proof:}
 See Appendix.
 
\begin{figure*}[tbh!]
    \centering
\includegraphics[width=\linewidth, height=3.0in]{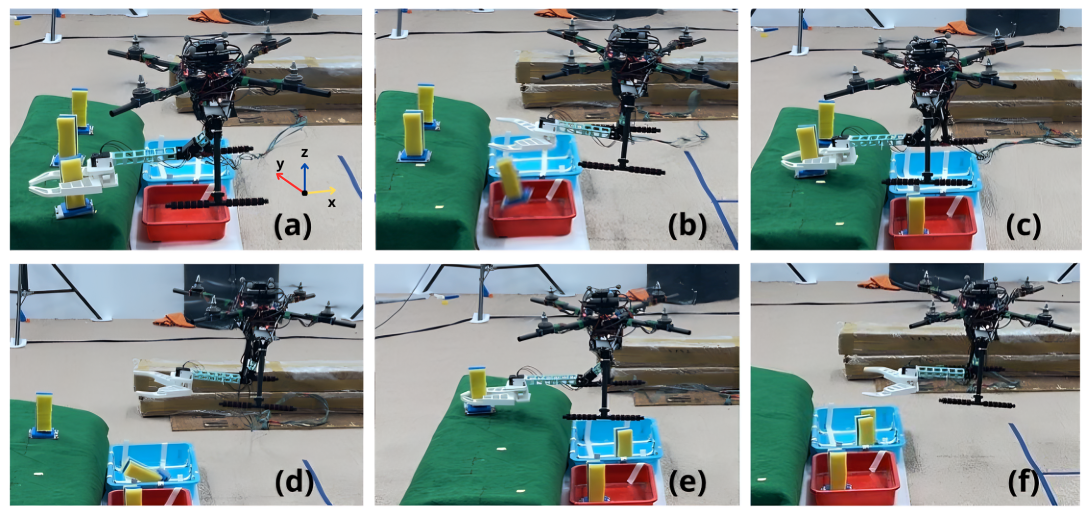} 
    \vspace{-20pt}
    \caption{\footnotesize Sequence of operations of the aerial manipulator during the experiment with the proposed controller: (a) takeoff from the ground and pick the first payload; (b) drop the first payload; (c) move and pick the second payload; (d) drop the second payload; (e) pick the third payload (f) drop the third payload.}
    \label{fig:exp_snap_1}
\end{figure*}

{\color{black}

 \begin{remark}[Guidelines to select $\vartheta$] \label{remark_switching_law}
In practice, some knowledge of the average number of regime changes $N_\sigma(T_1, T_2)$ in an interval $T_2-T_1$ is available, resulting in some ADT threshold: in the scenario considered in this work, it depends on the frequency for picking/dropping a payload. Then, in order to guarantee stable manipulation under such average number of system changes, the various gain parameters should be designed so that $\mu$ in (\ref{mu}) and $\varrho$ in (\ref{mu_def}) result in $\vartheta$ smaller than the ADT threshold, i.e., (\ref{eq:sw_law}) is validated.
\end{remark}

\begin{remark}[Choice of gains and trade-offs] \label{remark_gain_choice}
It can be noted from (\ref{mu_def}) and (\ref{new}) that high values of ($\boldsymbol  \Lambda_{p\sigma}, \boldsymbol  \Lambda_{q\sigma}, \boldsymbol  \Lambda_{\alpha\sigma}$) and of ($\nu_{p i\sigma}, \nu_{q i\sigma}, \nu_{\alpha i\sigma}$) can lead to faster error convergence, at the cost of higher control input, see (\ref{tau_j}). Therefore, the gains must be selected as per application requirement balancing between performance and control input demand. The matrices $\bar{\boldsymbol{M}}_{pp\sigma}, \bar{\boldsymbol{M}}_{qq\sigma}, \bar{\boldsymbol{M}}_{\alpha\alpha\sigma}$ can be selected based on some knowledge of the terms in the diagonal of $\boldsymbol{M}$. Because these diagonal terms represent the inertia of the individual subsystems and not the inertial couplings, they are much easier to model \cite{arleo2013control}.
\end{remark}

}

\begin{algorithm}[!h]
 \caption{ {Design steps of the proposed controller}}
 \textbf{Offline Parameter Design Phase:}
 \begin{itemize}
     \item \textbf{Design matrix gains:} Define the positive definite gain matrices $(\boldsymbol \Phi_p,\bar{\boldsymbol M}_{pp\sigma}, \boldsymbol \Lambda_{p\sigma})$, $(\boldsymbol \Phi_q, \bar{\boldsymbol M}_{qq\sigma}, \boldsymbol \Lambda_{q\sigma})$, and $(\boldsymbol \Phi_\alpha, \bar{\boldsymbol M}_{\alpha \alpha \sigma}, \boldsymbol \Lambda_{\alpha\sigma})$. 
     
     \item \textbf{Design class of switching signals:} Following Remark \ref{remark_switching_law}, calculate $\vartheta $ as 
per \eqref{eq:sw_law}. Stability is guaranteed for any switching signal belonging to such ADT switching. 

\item \textbf{Define gains for adaptive laws:} \textbf{(i)}: Define gains $\nu_{pi\sigma}$, $\nu_{qi\sigma}$ and 
{\color{black} $\nu_{\alpha i \sigma}$}, $i = 0, 1, 2, 3$.  
\textbf{(ii)}: Define $\delta_{p\sigma}$, $\delta_{q\sigma}$, $\delta_{\alpha \sigma}$. 
\textbf{(iii)}: Design $\epsilon_{p\sigma}$, $\epsilon_{q\sigma}$ $\epsilon_{\alpha \sigma}$ and derive {\color{black} $\varrho_{p\sigma}$, $\varrho_{q\sigma}$, $\varrho_{\alpha \sigma}$.} 
 \end{itemize}

\textbf{Online Control Phase:}
\begin{itemize}
    \item \textbf{Evaluate error variables:} Based on the design parameters in Step 1, calculate $s_p$, $s_q$, $s_\alpha$ according to \eqref{eq:s_j}. Also calculate $\xi$ as in \eqref{err}. 

\item \textbf{Design switching signal:} Evaluate the switching signal $\sigma(t)$ based on the current operational regime. 

    \item  \textbf{Evaluate adaptive gains:} As per \eqref{adaptive_law_j} and based on the current $\sigma(t)$, calculate the following adaptive gains \textbf{(i)} ($\hat{K}_{pi\sigma}$, $\hat{K}_{qi\sigma}$ $\hat{K}_{\alpha i\sigma}$), $({\zeta}_{p \sigma}$, ${\zeta}_{q \sigma}$ ${\zeta}_{\alpha \sigma})$ and  (${\gamma}_{p\sigma}$, ${\gamma}_{q\sigma}$ ${\gamma}_{\alpha \sigma}$), $i = 0, 1, 2, 3$
for the active regime and \textbf{(ii)} ($\hat{K}_{pi \bar\sigma}$, $\hat{K}_{qi \bar\sigma}$ $\hat{K}_{\alpha i \bar\sigma}$), $({\zeta}_{p \bar\sigma}$, ${\zeta}_{q \bar\sigma}$ ${\zeta}_{\alpha \bar\sigma})$ and  (${\gamma}_{p \bar\sigma}$, ${\gamma}_{q \bar\sigma}$ ${\gamma}_{\alpha \bar\sigma}$), $i = 0, 1, 2, 3$ for the inactive regime.

\item \textbf{Evaluate control inputs:} Based on the current $\sigma(t)$, calculate the control inputs $\tau_{p\sigma}$, $\tau_{q\sigma}$ and $\tau_{\alpha\sigma}$ via \eqref{tau_j}. 
\end{itemize}

 \end{algorithm}

The design steps of the proposed switched adaptive controller are summarized in Algorithm 1.
{\color{black} 
For tuning purposes, the following sequential tuning procedure can be adopted. First, the term in (\ref{del_j}) can be disabled ($\Delta\tau_{j\sigma}=0$), and $\bar{\boldsymbol M}_{jj\sigma}$ can be set from some nominal knowledge as in Remark \ref{remark_gain_choice}. The resulting nominal controller 
\[
\tau_{j\sigma} =
-\boldsymbol\Lambda_{j\sigma}\boldsymbol\Phi_j e_j
-\left(\boldsymbol\Lambda_{j\sigma}
+\bar{\boldsymbol M}_{jj\sigma}\boldsymbol\Phi_j\right)\dot e_j
+\bar{\boldsymbol M}_{jj\sigma}\ddot j_d,
\]
is then in a classical PD form with gains $
K_{p,j\sigma}=\boldsymbol\Lambda_{j\sigma}\boldsymbol\Phi_j,
\quad
K_{d,j\sigma}=\boldsymbol\Lambda_{j\sigma}
+\bar{\boldsymbol M}_{jj\sigma}\boldsymbol\Phi_j$, for which standard PD tuning rules apply. After tuning the nominal gains, the adaptive parameters can be enabled and moderately tuned. 
{\color{black}Here, $K_{p,j\sigma}$ and $K_{d,j\sigma}$ govern the nominal error convergence, whereas $\nu_{ji\sigma}$ is the primary tuning parameter balancing adaptation speed and sensor noise rejection.}
}

\section{Experimental Results and Analysis}
For experimental purpose, an aerial manipulator setup is created using a Tarrot-650 quadrotor frame {\color{black} (with KV380 SunnySky V4006 brushless motors; {\color{black}14-inch propeller}; 6S battery) } and a 2R serial-link manipulator system (with Dynamixel XM430-W210-T motors). The end-effector is a custom 3-D printed gripper actuated by a Dynamixel motor for payload pick-and-place operation. The overall setup weighs {\color{black}$3$~kg} (approx.). A U2D2 Power Hub Board is used to power the manipulator and the gripper. Raspberry Pi-4 is used as a processing unit which uses a U2D2 communication converter.  Sensor data from OptiTrack motion capture system and IMU were used to measure the necessary pose (position, attitude), velocity and acceleration feedback of the quadrotor. For the manipulator, the joint angular position and velocity are measured by the Dynamixel motors, while the accelerations are computed numerically.

To achieve attitude control, the tracking error in orientation is defined as~\cite{mellinger2011minimum}:
\begin{align}
    e_q = \frac{1}{2} \left( R_d^{\top} R_B^W - (R_B^W)^{\top} R_d \right)^{\vee}, 
    \dot{e}_q = \dot{q} - R_d^{\top} R_B^W \dot{q}_d, \nonumber
\end{align}
where $(\cdot)^{\vee}$ denotes the \emph{vee} operator that maps an element of $SO(3)$ to $\mathbb{R}^3$.
The matrix $R_d \in SO(3)$ represents the desired rotation evaluated at $(\phi_d, \theta_d, \psi_d)$,
and $q_d$ is the corresponding desired attitude as defined in~\cite{mellinger2011minimum}.
{\color{black}
For implementation, velocity and acceleration signals are low-pass filtered to reduce IMU/OptiTrack noise and numerical differentiation errors. The leakage parameter $\nu_{ji\sigma}$ and boundary-layer parameter $\varpi_j$ regulate the trade-off between adaptation speed, noise sensitivity, and
chattering attenuation: namely,  $-\nu_{ji\sigma}\hat K_{ji\sigma}$ prevent excessive adaptive-gain growth, while $\varpi_j$ in $\Delta\tau_{j\sigma}$ mitigates chattering near $\|s_j\|=0$. The parameters used in the experiments for the proposed controller are listed in Table \ref{tab:control_parameters}.

The desired quadrotor and manipulator trajectories are pre-computed using \textit{mav\_trajectory\_generation}~\cite{mav_trajectory_generation} and \textit{Ruckig}~\cite{berscheid2021jerk}, respectively. These tools generate smooth, kinematically feasible references from waypoints and motion limits without requiring the aerial-manipulator dynamic model.
}

{\color{black}
Since off-center payloads can induce additional wrenches, the experiments were conducted by selecting payload masses, arm configurations, and reference trajectories that keep the required force/moment within the available actuation capabilities.

}

\subsection{Experimental Scenario and Parameter Selection}
To verify the performance of the proposed controller, an experimental scenario was created where the aerial manipulator is tasked to pick up three different payloads (of weight approx. {\color{black}$0.3$~kg, $0.4$~kg, and $0.5$~kg}) from a location and drop it to another designated place, sequentially (cf. Fig. \ref{fig:exp_snap_1}). The payload positions were fed to the quadrotor using the motion capture markers. The experimental scenario consists of the following sequence:
\begin{enumerate}[label=\alph*)]
    \item The quadrotor takes off from its origin ($x = 0, y = 0$) and ascends to a desired height of {\color{black}$z_d=1$~m}. Initially, the manipulator link angles are set to $\alpha_1 (0) = \ang{0}$ and $\alpha_2(0) = \ang{90}$. 
    \item The quadrotor moves towards the first payload location $x = -1,y = -0.5$ while the manipulator links rotate to the desired angles $\alpha_1^d = \alpha_2^d = \ang{45}$ to pick the first payload of {\color{black}$0.3$~kg} (cf. Fig. \ref{fig:exp_snap_1}(a)). 
    The quadrotor moves backward to the position $x = 0.6, y = -0.5$, with the manipulator adjusting to the desired angles $\alpha_1^d = \ang{0}, \alpha_2^d = \ang{45}$ and then drop the first payload (cf. Fig. \ref{fig:exp_snap_1}(b)). Note that different arm positions during payload pick-up and drop give different inertia distributions during regime changes $\sigma=1 \mapsto 2$ and $\sigma=2 \mapsto 1$.
    \item Following a similar procedure, the quadrotor retrieves the second payload of {\color{black}$0.4$~kg} from $x = -1, y = 0.0$ (cf. Fig. \ref{fig:exp_snap_1}(c)) to drop at $x = 0.6, y = 0.0$ (cf. Fig. \ref{fig:exp_snap_1}(d)) and the third payload of {\color{black}$0.5$~kg} from $x = -1, y = 0.5$ (cf. Fig. \ref{fig:exp_snap_1}(e)) to drop at $x = 0.6, y = 0.5$ (cf. Fig. \ref{fig:exp_snap_1}(f)). The manipulator follows the same desired link angles used for the first payload for the other two payloads.
\end{enumerate}
The above sequence corresponds to a switching signal $\sigma(t)$ as in Fig. \ref{switching}, where the transitions from $\sigma=1$ to $\sigma=2$ and $\sigma=2$ to $\sigma=1$ indicate instances of payload pick-up and drop respectively. 

\begin{figure}[!t]
\includegraphics[width=0.4\textwidth, height=1.8in]{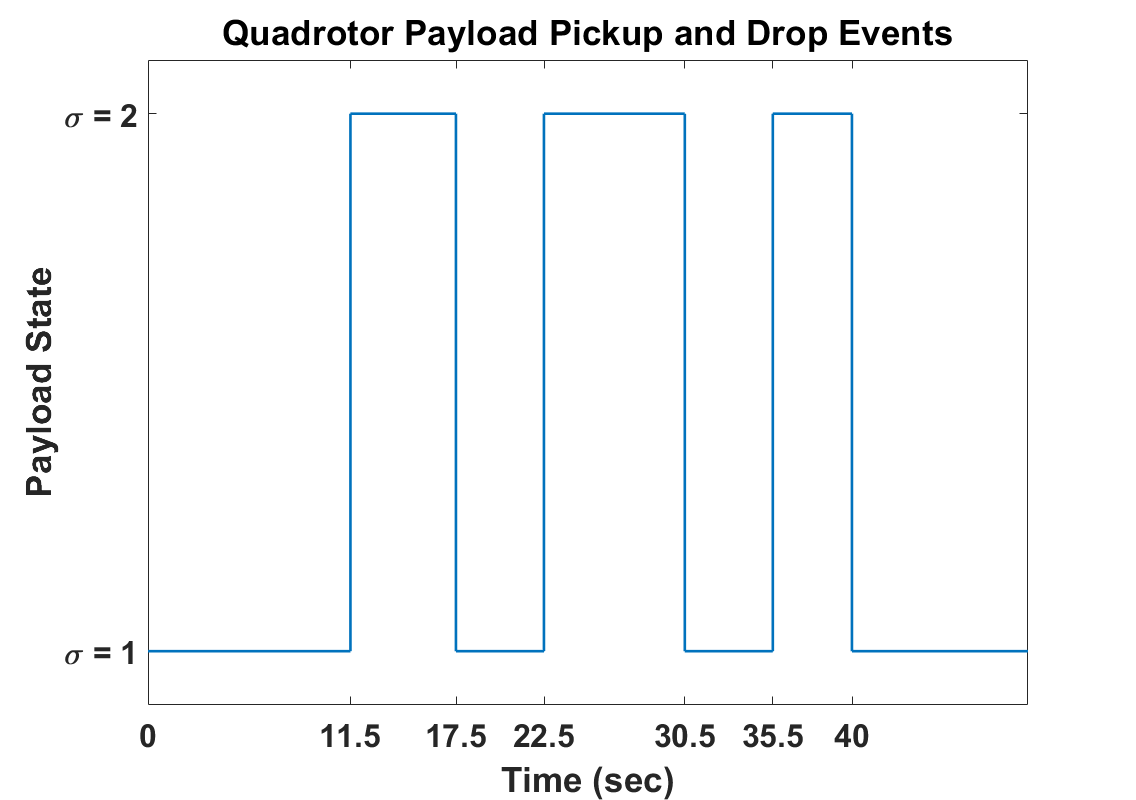}
    \centering
    \caption{The switching regimes used in both the scenarios.}
    \label{switching}
\end{figure}

To verify the significance of the proposed switched adaptive control, we compare it against two baseline non-switched controllers: (i) a non-switched adaptive control (abbreviated as NS-AC) version of the proposed scheme, obtained by only considering the control for regime $\sigma=1$, and (ii) a disturbance observer-based controller (abbreviated as DOC). 

The selected gains yield $\mu=\max\{1.5,1.667,2\}=2$ and $\kappa=0.2$,
and therefore the required ADT threshold is 
$\vartheta
=
\frac{\ln\mu}{\kappa}
=
{\color{black}3.47~\mathrm{s}}$.
The six payload pick-up and release transitions over the
{\color{black}$50$~s} experimental horizon give the empirical average
inter-switching time
$
{\vartheta}_{\mathrm{actual}}
=\frac{T_2-T_1}{N_\sigma(T_1,T_2)}
=\frac{50}{6}
={\color{black}8.33~\mathrm{s}}$.
Since
$\vartheta < {\vartheta}_{\mathrm{actual}}$,
the experimental switching sequence satisfies the considered
ADT-based sufficient stability condition.
The switching signal in Fig. \ref{switching} belongs to the class of switching signals for which the system can be stabilized. The control parameters and estimated coupling parameters required for DOC are selected according to the disturbance model presented in \cite{10722859} after considering nominal parametric information from the system (e.g., mass, geometry, etc.). Note that NS-AC and the proposed controller do not require any prior knowledge of inertial couplings. Hence, this comparative study allows to observe the effects on control performance of the switching dynamics and of the uncertainty in inertial couplings.

\begin{table}[htbp]
\centering
\footnotesize
\caption{Design Parameters for the Proposed Controller}
\begin{tabular}{|p{0.45\textwidth}|}
\hline
\textbf{Quadrotor Position Control} \\
\hline

$\bar{\boldsymbol{M}}_{pp{\color{black}1}} = \boldsymbol I_3$,
$\bar{\boldsymbol{M}}_{pp{\color{black}2}} = 1.5\boldsymbol I_3$,
${\boldsymbol \Phi}_{p} = \text{diag}\{1.5,1.5,1.1\}$,
${\boldsymbol \Lambda}_{p{\color{black}1}} = \text{diag}\{4,4,4\}$,
${\boldsymbol \Lambda}_{p{\color{black}2}} = \text{diag}\{5,5,6\}$

$\nu_{p0{\color{black}1}} = \nu_{p1{\color{black}1}} =
\nu_{p2{\color{black}1}} = \nu_{p3{\color{black}1}} =
\nu_{p0{\color{black}2}} = \nu_{p1{\color{black}2}} =
\nu_{p2{\color{black}2}} = \nu_{p3{\color{black}2}} = 2.0$,
${\delta}_{p{\color{black}1}} = {\delta}_{p{\color{black}2}} =
{\epsilon}_{p{\color{black}1}} = {\epsilon}_{p{\color{black}2}} = 0.001$,

$\hat{K}_{p0{\color{black}1}}(0) =
\hat{K}_{p1{\color{black}1}}(0) =
\hat{K}_{p2{\color{black}1}}(0) =
\hat{K}_{p3{\color{black}1}}(0) =
\hat{K}_{p0{\color{black}2}}(0) =
\hat{K}_{p1{\color{black}2}}(0) =
\hat{K}_{p2{\color{black}2}}(0) =
\hat{K}_{p3{\color{black}2}}(0)=0.01$,
$\zeta_{p{\color{black}1}}(0) =
\zeta_{p{\color{black}2}}(0) =
\gamma_{p{\color{black}1}}(0) =
\gamma_{p{\color{black}2}}(0)=0.1$,

$\varpi_p=0.1$\\
\hline
\hline

\textbf{Quadrotor Attitude Control} \\
\hline

$\bar{\boldsymbol{M}}_{qq{\color{black}1}} =
0.015\boldsymbol I_3$,
$\bar{\boldsymbol{M}}_{qq{\color{black}2}} =
0.025\boldsymbol I_3$,
${\boldsymbol \Phi}_{q}=\text{diag}\{5,5,4\}$,
${\boldsymbol \Lambda}_{q{\color{black}1}} =
\text{diag}\{1.2,1.2,0.20\}$,
${\boldsymbol \Lambda}_{q{\color{black}2}} =
\text{diag}\{1.4,1.4,0.25\}$

$\nu_{q0{\color{black}1}} = \nu_{q1{\color{black}1}} =
\nu_{q2{\color{black}1}} = \nu_{q3{\color{black}1}} =
\nu_{q0{\color{black}2}} = \nu_{q1{\color{black}2}} =
\nu_{q2{\color{black}2}} = \nu_{q3{\color{black}2}} = 10.0$,
${\delta}_{q{\color{black}1}} = {\delta}_{q{\color{black}2}} =
{\epsilon}_{q{\color{black}1}} = {\epsilon}_{q{\color{black}2}} = 0.0001$,

$\hat{K}_{q0{\color{black}1}}(0) =
\hat{K}_{q1{\color{black}1}}(0) =
\hat{K}_{q2{\color{black}1}}(0) =
\hat{K}_{q3{\color{black}1}}(0) =
\hat{K}_{q0{\color{black}2}}(0) =
\hat{K}_{q1{\color{black}2}}(0) =
\hat{K}_{q2{\color{black}2}}(0) =
\hat{K}_{q3{\color{black}2}}(0)=0.001$,
$\zeta_{q{\color{black}1}}(0) =
\zeta_{q{\color{black}2}}(0) =
\gamma_{q{\color{black}1}}(0) =
\gamma_{q{\color{black}2}}(0)=0.01$,

$\varpi_q=1.0$\\
\hline
\hline

\textbf{Manipulator Control} \\
\hline

$\bar{\boldsymbol{M}}_{\alpha\alpha{\color{black}1}}
=0.05\boldsymbol I_2$,
$\bar{\boldsymbol{M}}_{\alpha\alpha{\color{black}2}}
=0.1\boldsymbol I_2$,
${\boldsymbol \Phi}_{\alpha}=1.5\boldsymbol I_2$,
${\boldsymbol \Lambda}_{\alpha{\color{black}1}}
=\boldsymbol I_2$,
${\boldsymbol \Lambda}_{\alpha{\color{black}2}}
=2\boldsymbol I_2$,

$\nu_{\alpha0{\color{black}1}} =
\nu_{\alpha1{\color{black}1}} =
\nu_{\alpha2{\color{black}1}} =
\nu_{\alpha3{\color{black}1}} =
\nu_{\alpha0{\color{black}2}} =
\nu_{\alpha1{\color{black}2}} =
\nu_{\alpha2{\color{black}2}} =
\nu_{\alpha3{\color{black}2}} = 5.0$,
${\delta}_{\alpha{\color{black}1}} =
{\delta}_{\alpha{\color{black}2}} =
{\epsilon}_{\alpha{\color{black}1}} =
{\epsilon}_{\alpha{\color{black}2}} = 0.0001$,

$\hat{K}_{\alpha0{\color{black}1}}(0) =
\hat{K}_{\alpha1{\color{black}1}}(0) =
\hat{K}_{\alpha2{\color{black}1}}(0) =
\hat{K}_{\alpha3{\color{black}1}}(0) =
\hat{K}_{\alpha0{\color{black}2}}(0) =
\hat{K}_{\alpha1{\color{black}2}}(0) =
\hat{K}_{\alpha2{\color{black}2}}(0) =
\hat{K}_{\alpha3{\color{black}2}}(0)=0.001$,
$\zeta_{\alpha{\color{black}1}}(0) =
\zeta_{\alpha{\color{black}2}}(0) =
\gamma_{\alpha{\color{black}1}}(0) =
\gamma_{\alpha{\color{black}2}}(0)=0.01$,

$\varpi_\alpha=1.0$\\
\hline
\end{tabular}
\label{tab:control_parameters}
\end{table}

\subsection{Results and Analysis}
\begin{figure}[!h]
\includegraphics[width=0.46\textwidth]{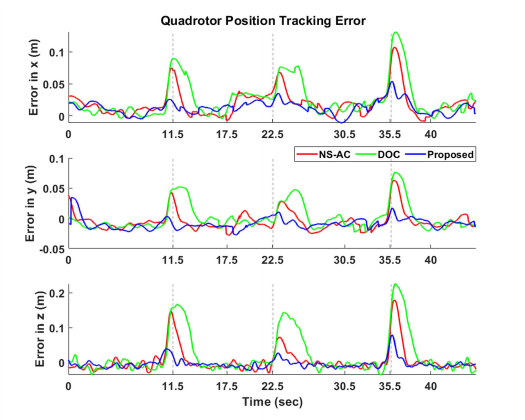}
    \centering
    \caption{Comparison of quadrotor position tracking error with various controllers. {\color{black} Vertical dotted lines denote payload pickup instants.} }
    \label{fig1}
\end{figure}

\begin{figure}[!h]
\includegraphics[width=0.46\textwidth]{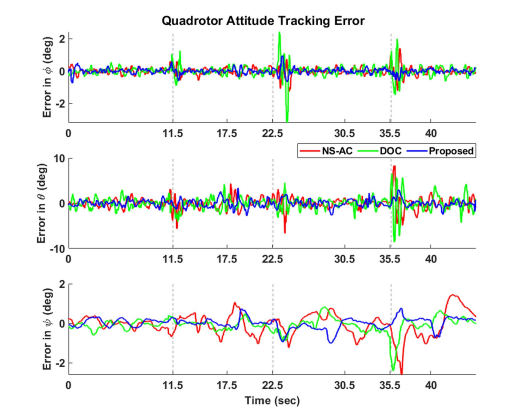}
    \centering
    \caption{Comparison of quadrotor attitude tracking error with various controllers. {\color{black} Vertical dotted lines denote payload pickup instants.}}
    \label{fig2}
\end{figure}

\begin{figure}[!h]
\includegraphics[width=0.46\textwidth]{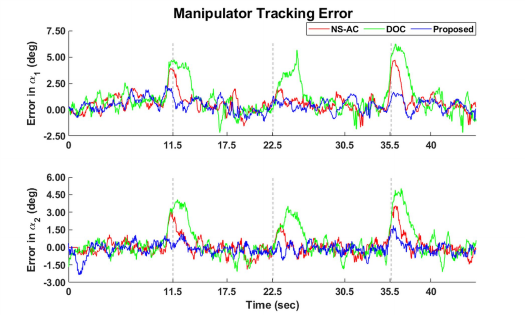}
    \centering
    \caption{Comparison of manipulator tracking error with various controllers. {\color{black} Vertical dotted lines denote payload pickup instants.}}
    \label{fig3}
\end{figure}

\begin{table}[t]
\footnotesize
\renewcommand{\arraystretch}{1.2}
\caption{\small Tracking performance comparison (RMS)}
		\centering
{
\scalebox{0.8}{	
\begin{tabular}{c|ccc|ccc|cc}
\hline
\multirow{2}{*}{} & \multicolumn{3}{c|}{Position Error (m)}                                   & \multicolumn{3}{c|}{Attitude Error (deg)}     
& \multicolumn{2}{c}{Arm Error (deg)}
      \\ \cline{2-9} 
                  & $x$  & $y$  & $z$        & $\phi$ & $\theta$  & $\psi$   & $\alpha_1$         & $\alpha_2$         \\ \hline
DOC   & $0.093 $& $0.041 $& $0.153$& $1.603$& $5.142$& $1.987$& $3.881$& $3.091$\\ \hline
NS-AC   & $0.079 $& $0.038$& $0.105 $& $1.477$& $4.097$& $1.128$& $2.225 $& $2.016$\\ \hline
Proposed   & $\mathbf{0.031}$& $\mathbf{0.019}$& $\mathbf{0.052}$& $\mathbf{1.128 }$& $\mathbf{2.018}$& $\mathbf{1.033 }$& $\mathbf{1.107}$& $\mathbf{1.015}$\\ \hline
\end{tabular}}}
\label{table_performance}

\end{table}

The performance of the various controllers is shown in
Figs. \ref{fig1}-\ref{fig3}. 
A clear trend emerges at approximately {\color{black}$t=11.5$~s, $22.5$~s, and $35.5$~s} in Fig. \ref{fig1} and \ref{fig3}, where sharp error spikes appear in NS-AC and DOC due to payload pickup events. While NS-AC and DOC perform comparably to the proposed switched adaptive controller without a payload (regime corresponding to $\sigma = 1$), their errors increase significantly when mass and inertia suddenly change after payload pick-up (regime corresponding to $\sigma = 2$).
 DOC relies on a precomputed disturbance model, while it has been explained, the inertial couplings are very difficult to model and give rise to state-dependent uncertainties. Consequently, when the actual disturbances deviate from its assumed model, DOC fails to correct mismatches, leading to higher peaks and oscillatory errors compared to the other controllers. NS-AC lacks an explicit switching mechanism and assumes a single continuously evolving model. This results in a slower decay in error after each spike, as it takes time to compensate for abrupt payload variations. However, having the capability to tackle the unknown inertial couplings, it ends up outperforming DOC.

In contrast, the proposed switched adaptive controller effectively adjusts itself for each dynamic transition (regimes for $\sigma = 1, 2$) while handling unknown inertial couplings. This results in lower peak tracking errors and faster recovery after payload pickups, enhancing both transient performance and overall robustness. The lower Root Mean Square (RMS) errors in Table \ref{table_performance}  validate the effectiveness of the proposed switching framework in maintaining stability under dynamic transitions. 
{ \color{black}
The computation time for all methods is reported in Table \ref{table_compute}.

\begin{figure}[!h]
\includegraphics[width=0.48\textwidth]{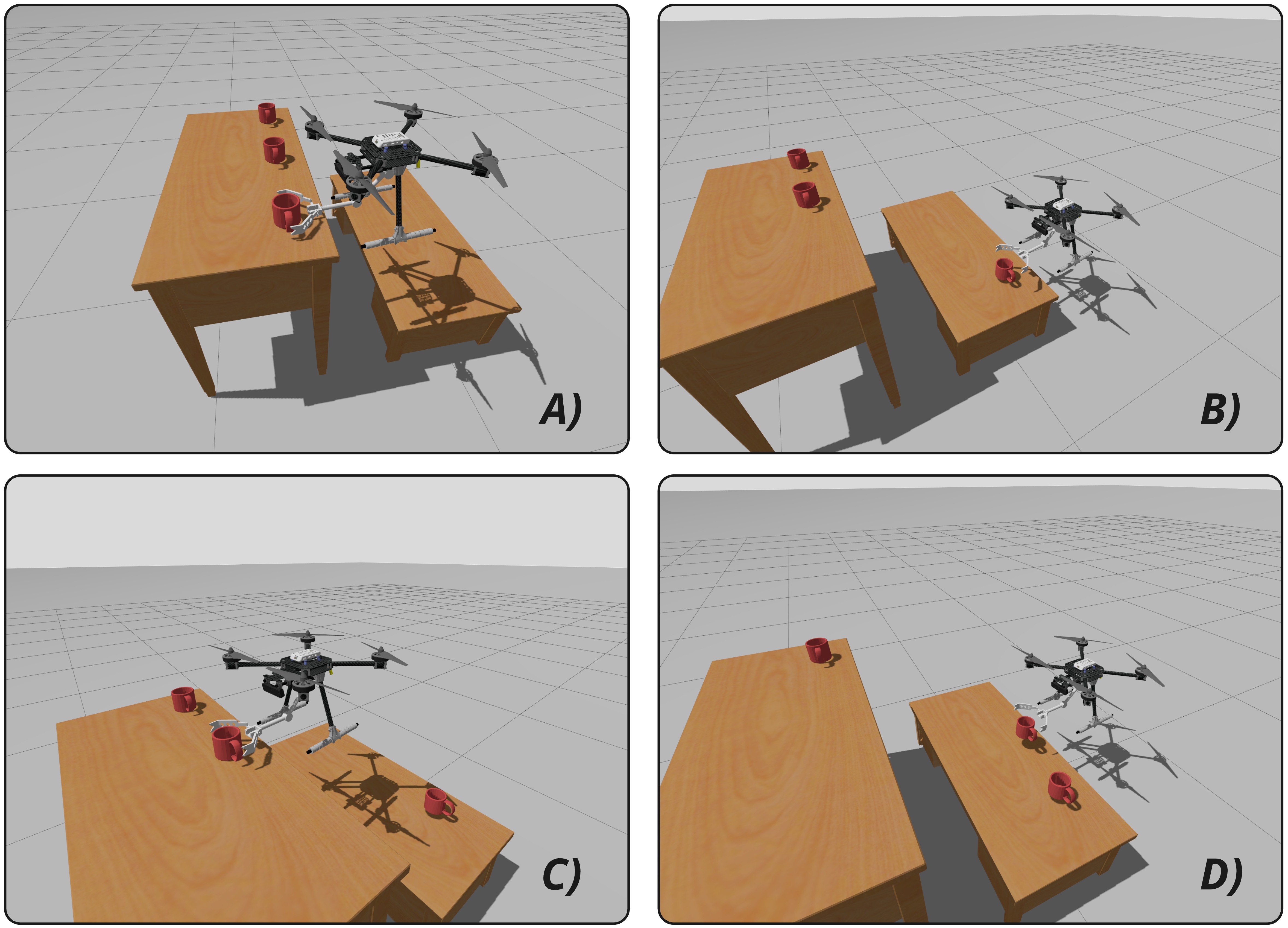}
    \centering
    \caption{{\color{black}The experimental simulation setup showing the picking (Fig A,C) and placing (Fig B,D) of an object.}}
    \label{fig_simulation}
    
\end{figure}

\begin{table}[t]
\footnotesize
\renewcommand{\arraystretch}{1.2}
{\color{black}\caption{\small Onboard computation time comparison}
\centering
\scalebox{0.9}{
\begin{tabular}{c|ccc}
\hline
\textbf{Metric} & \textbf{NS-AC} & \textbf{DOC} & \textbf{Proposed} \\
\hline
Computation time (ms) & 0.40 & 0.55 & 0.50 \\
\hline
\end{tabular}}\label{table_compute}

\parbox{0.9\linewidth}{\footnotesize
\textit{Note:} Computation time denotes the average execution time of one
control-loop iteration on the Raspberry Pi~4 onboard processor (C++).}}
\label{tab:onboard_computation_time}
\end{table}

{\color{black}

\subsection{Simulated Ablation and Sensitivity Studies}
\label{sec:ablation}

\begin{figure}[!h]
\includegraphics[width=0.46\textwidth, height= 4.1in]{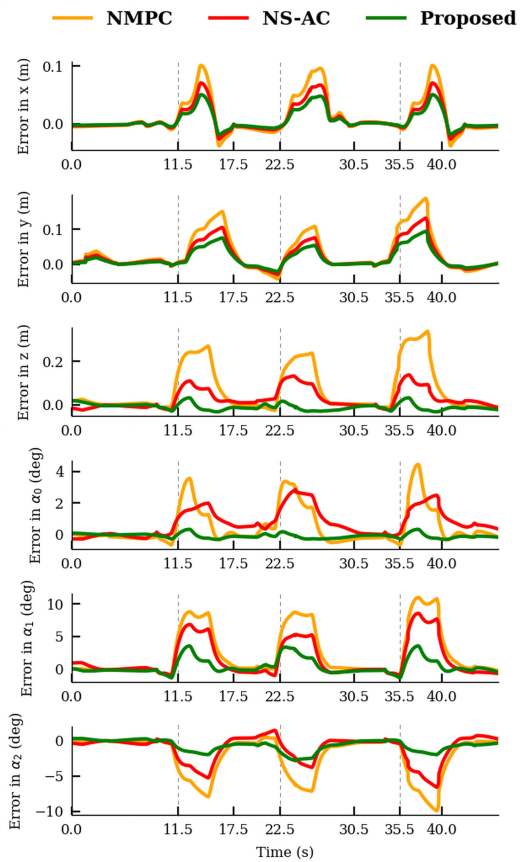}
    \centering
    \caption{{\color{black}Comparison of manipulator tracking error with various controllers in simulation. {\color{black} Vertical dotted lines denote payload pickup instants.}}}
    \label{fig9}
\end{figure}

\begin{figure}[!h]
\includegraphics[width=0.45\textwidth, height= 4.1in]{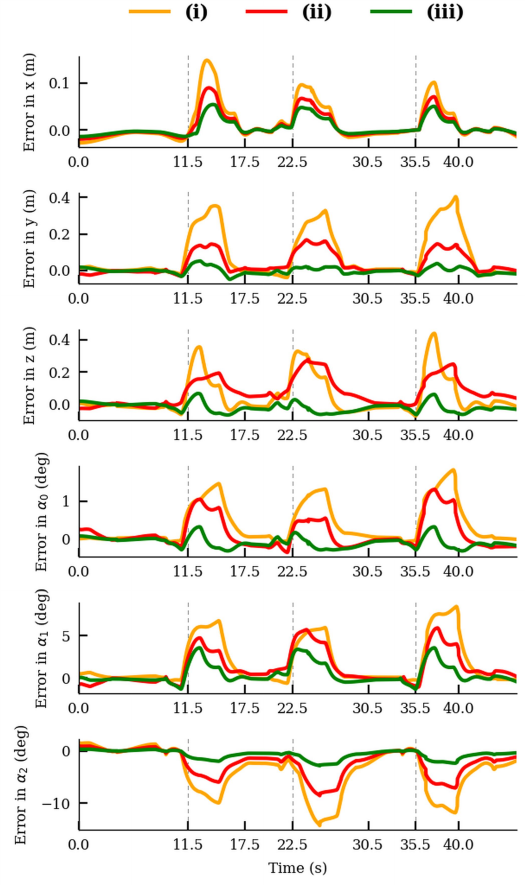}
    \centering
    \caption{{\color{black}Tracking performance for (i) clearly violated (ii) slightly violated, and (iii) satisfied ADT conditions. {\color{black} Vertical dotted lines denote payload pickup instants.}}}
    \label{fig10}
\end{figure}

To further assess the proposed switched adaptive design under different conditions, we conducted an additional simulation study using an aerial manipulator with a 3-DoF arm in PX4-SITL as shown in Fig. \ref{fig_simulation}. The task setup follows the same payload pick-and-place scenario shown in Fig.~\ref{fig:exp_snap_1}. The physical parameters of the simulated system, including the mass, geometry, inertial properties, actuator configuration, and payload characteristics, were selected to closely emulate the experimental setup. Consequently, {\color{black}the same controller design parameters reported in Table~\ref{tab:control_parameters} were used, with the manipulator gains replicated for the additional DoF.}

\subsubsection{Tracking performance}
The proposed controller was compared with two baselines: the previous non-switched adaptive controller (NS-AC) and the nonlinear model predictive control (NMPC) in \cite{hanover2021performance}. The NMPC benchmark was implemented using the ACADO Toolkit, which generates
efficient C/C++ code for real-time nonlinear optimal control. 
NMPC uses a non-switched system model tuned for the no-payload condition, where the coupling terms were set to zero. This allows us to isolate the effect of state-dependent coupling forces induced during payload pick-up. 
The tracking results are shown in Fig.~\ref{fig9}, while the average computation time are reported in Table~\ref{tab:simulation_computation_time}. The proposed controller achieves better tracking and faster convergence than both baselines. Compared with NS-AC, it benefits from regime-dependent adaptive gains during payload transitions. NMPC, instead, depends on a model designed for the no-payload
condition, and cannot compensate for switching-induced uncertainties, leading to slower convergence and a persistent tracking offset. The post-transition settling performance follows the order
$\mathrm{NMPC}<\mathrm{NS\mbox{-}AC}<\mathrm{Proposed}$.

\begin{table}[t]
\footnotesize
\renewcommand{\arraystretch}{1.2}
\caption{{\color{black}\small Estimated controller computation time on a 3.0~GHz processor laptop (C++)}}
\centering
{\color{black}\begin{tabular}{c|cccc}
\hline
\textbf{Controller} & \textbf{NS-AC}  &
\textbf{Proposed} & \textbf{NMPC} \\
\hline
Computation time (ms) & 0.020 &  0.023 & 3.6 \\
\hline
\end{tabular}
\label{tab:simulation_computation_time}}
\end{table}

\subsubsection{Study on Different ADT Conditions}

To investigate the effect of satisfying or violating the ADT condition,
the same switching sequence is retained in all cases. The ADT condition in
\eqref{eq:sw_law} requires $
\vartheta <
\frac{T_2-T_1}{N_\sigma(T_1,T_2)}
=
\frac{50}{6}
=
8.33~\mathrm{s}$.
For the nominal 3-DoF case, the same controller gains as those used in the
real experiments are retained. Accordingly, Table~\ref{tab:control_parameters}
gives $
\mu=\max\{1.5,1.667,2\}=2$,
and, with $\kappa=0.2$, $
\vartheta=\frac{\ln(2)}{0.2}=3.47~\mathrm{s}$,
which satisfies the ADT condition. For the two additional cases, only the
design matrix {\color{black}$\bar{\boldsymbol M}_{\alpha\alpha2}$} is changed to
{\color{black}$0.274\boldsymbol I_3$ and $0.585\boldsymbol I_3$}, yielding
$\mu=5.48$ and $11.7$, and hence
$\vartheta=8.51~\mathrm{s}$ and $12.30~\mathrm{s}$, respectively. These
values correspond to a slight violation near the theoretical boundary and
a clear violation of the ADT condition. The corresponding tracking results
are shown in Fig.~\ref{fig10}.
The tracking errors increase as the degree of violation becomes larger. Note that the ADT condition in \eqref{eq:sw_law} is a sufficient condition for stability, so that, its violation removes the theoretical stability guarantee but does not necessarily cause instability. Hence, the proposed method is robust to different ADT conditions.

\subsubsection{Noise-Sensitivity Analysis}

For noise attenuation, each measured signal $z\in{s_j,\xi,\ddot{\chi}}$ was filtered using a first-order low-pass filter,
$
z_f[k]=(1-\beta)z_f[k-1]+\beta z[k], ~~
\beta=\frac{T_s}{\tau_f+T_s}$.
The controller was implemented with {\color{black}$T_s=0.005~\mathrm{s}$} and
{\color{black}$\tau_f=0.005~\mathrm{s}$}, giving $\beta=0.5$,
{\color{black}corresponding to an effective $-3$~dB cutoff frequency of approx.
$23$~Hz.} To evaluate noise sensitivity, Gaussian noise was injected into the measured signals, with position-noise standard deviations of $0.01$, $0.03$, and $0.05~\mathrm{m}$ for the low-, moderate-, and high-noise cases, respectively, i.e.,
$ y_{\mathrm{meas}}(t)=y(t)+\eta(t),~ 
\eta(t)\sim \mathcal{N}(0,\sigma_\eta^2).$ The results in Table \ref{table:noise_performance} show that the tracking errors remain bounded and, as expected, the bound increases with increasing noise.

\begin{table}[t]
\footnotesize
\renewcommand{\arraystretch}{1.2}
\caption{{\color{black}\small Sensitivity of tracking performance to noise (RMS)}}
\centering
\resizebox{\columnwidth}{!}{
{\color{black}
\begin{tabular}{c|ccc|ccc|ccc}
\hline
\multirow{2}{*}{\textbf{Noise}}
& \multicolumn{3}{c|}{Position Error (m)}
& \multicolumn{3}{c|}{Attitude Error (deg)}
& \multicolumn{3}{c}{Arm Error (deg)}
\\
\cline{2-10}
& $x$ & $y$ & $z$
& $\phi$ & $\theta$ & $\psi$
& $\alpha_1$ & $\alpha_2$ & $\alpha_3$
\\
\hline
Low
& 0.012 & 0.009 & 0.021
& 0.21 & 0.18 & 0.24
& 0.487 & 0.408 & 0.491
\\ \hline
Moderate
& 0.016 & 0.011 & 0.028
& 0.29 & 0.25 & 0.32
& 0.566 & 0.510 & 0.550
\\ \hline
High
& 0.021 & 0.013 & 0.035
& 0.41 & 0.36 & 0.45
& 0.735 & 0.678 & 0.741
\\
\hline
\end{tabular}
}}
\label{table:noise_performance}
\end{table}

}}

\section{Conclusion}
This paper presented a switched adaptive control framework to handle unknown inertial couplings, state-dependent uncertainties, and discontinuous variations in aerial manipulator operations. It has been shown that the proposed switched adaptive control framework ensures regime-dependent adaptation, enhancing robustness across different operational regimes. 
Both simulated and real experiments have validated the effectiveness of regime-aware adaptation for precise and reliable aerial manipulation.

\section{Appendix: Proof of Theorem 1}
\label{sec:app}

From the adaptive laws and initial conditions in \eqref{adaptive_law_j}, it can be verified that $\hat{K}_{ji\sigma}(t)\geq 0$ $\forall t>0$ $i=0,\cdots,3, j= p,q, \alpha$ and $\exists \ushort{\zeta}_{j\sigma}, \ushort{\gamma}_{j\sigma} \in \mathbb{R}^{+}$ such that $\forall t \geq t_0$
\begin{align} \label{adaptive_gain_bound}
0 < \ushort{\zeta}_{j\sigma} \leq \zeta_{j\sigma} (t) <  \bar{\zeta}_{j\sigma}; ~~ 0 < \ushort{\gamma}_{j\sigma} \leq \gamma_{j\sigma} (t) <  \bar{\gamma}_{j\sigma}. 
\end{align}

The closed-loop stability is analyzed using the following Lyapunov function candidate:
\begin{align}
&V(t) = V_{{p}}(t) + V_{{q}}(t) +  V_{\alpha}(t), \label{complete_lyap} 
\end{align}
where, for $j=p,q,\alpha$, we define $V_p$, $V_q$, $V_\alpha$ as
\begin{align}
&V_j(t) = \frac{1}{2} s_j^{\top}(t) \bar{\boldsymbol M}_{jj\sigma(t)} s_j(t) + \nonumber \\ 
&{\color{black}\sum_{\sigma=1}^{2}} \left(\frac{\zeta_{j\sigma}(t)}{\ushort{\zeta}_j} + \frac{\gamma_{j\sigma}(t)}{\ushort{\gamma}_j} +   \sum_{i=0}^{3}\frac{(\hat{K}_{ji\sigma}(t) - K_{ji\sigma}^{*})^2}{2} \right) \label{lyap_j} ,
\end{align}
with $\ushort{\gamma}_j = \min_{\sigma \in \Omega} \{\ushort{\gamma}_{j\sigma} \}$ and $\ushort{\zeta}_j = \min_{\sigma \in \Omega} \{\ushort{\zeta}_{j\sigma} \}$.

Let regime ${\sigma({t_{l+1}^{-}})}$ be active when $t \in [t_l~~t_{l+1})$ and regime ${\sigma({t_{l+1}})}$ be active when $t \in [t_{l+1}~~t_{l+2})$. Then, without loss of generality, the stability analysis is carried out by observing the behaviour of $V(t)$ in two cases: (A) at the switching instant $t_{l+1},~ l \in \mathbb{N}^{+}$ and (B) between two consecutive switching instants when $t \in [t_l~~t_{l+1})$. In the following, we conduct such investigations.
\subsection{Behaviour of $V(t)$ at the switching instant}
We first analyze the individual Lyapunov functions ${V}_p$, ${V}_q$, and ${V}_\alpha$, and then combine them to assess the overall stability. For $j \in {p, q, \alpha}$, from \eqref{lyap_j}, we have before switching,
\begin{align} 
&V_j({t_{l+1}^{-}}) = \frac{1}{2} s_j^{\top}({t_{l+1}^{-}}) \bar{\boldsymbol M}_{jj\sigma({t_{l+1}^{-}})} s_j({t_{l+1}^{-}}) +   \nonumber \\
& {\color{black}\sum_{\sigma=1}^{2}}\left( \frac{\zeta_{j\sigma}({t_{l+1}^{-}})}{\ushort{\zeta}_j}+ \frac{\gamma_{j\sigma}({t_{l+1}^{-}})}{\ushort{\gamma}_j} +\sum_{i=0}^{3}\frac{(\hat{K}_{ji\sigma}({t_{l+1}^{-}}) - K_{ji\sigma}^{*})^2}{2} \right) , \nonumber
\end{align}
and, after switching,
\begin{align}
&V_j({t_{l+1}}) = \frac{1}{2} s_j^{\top}({t_{l+1}}) \bar{\boldsymbol M}_{jj\sigma({t_{l+1}})} s_j({t_{l+1}}) + \nonumber \\
& {\color{black}\sum_{\sigma=1}^{2}} \left( \frac{\zeta_{j\sigma}({t_{l+1}})}{\ushort{\zeta}_j}   + \frac{\gamma_{j\sigma}({t_{l+1}})}{\ushort{\gamma}_j} +\sum_{i=0}^{3}\frac{(\hat{K}_{ji\sigma}({t_{l+1}}) - K_{ji\sigma}^{*})^2}{2} \right). 
\nonumber \end{align}

{\color{black}Having the system no state jumps at switching
instants (or taking them to be negligible),} the variables $s_j$, $\zeta_{j\sigma}$, $\gamma_{j\sigma}$,
$\hat{K}_{ji\sigma}$ are continuous, and hence
$s_j({t_{l+1}^{-}}) = s_j({t_{l+1}})$,
$\zeta_{j\sigma}({t_{l+1}^{-}}) =  \zeta_{j\sigma}({t_{l+1}})$,
$\gamma_{j\sigma}({t_{l+1}^{-}}) = \gamma_{j\sigma}({t_{l+1}})$,
$(\hat{K}_{ji\sigma}({t_{l+1}^{-}}) - K_{ji\sigma}^{*}) =
(\hat{K}_{ji\sigma}({t_{l+1}}) - K_{ji\sigma}^{*})$
for all $\sigma \in \Omega$.
Further, from the inequalities
$ {\ushort{\varrho}_j} s_{j}^{\top}({t}) s_{j}({t})
\leq s_{j}^{\top}({t})
{\color{black}\bar{\boldsymbol M}_{jj\sigma(t)}}
s_{j}({t})
\leq {\bar{\varrho}_j} s_{j}^{\top}({t}) s_{j}({t})$,
one has
\begin{align}
  &V_{j}({t_{l+1}}) - V_{j}({t_{l+1}^{-}}) \nonumber \\
  &=  \frac{1}{2} s_{j}^{\top}({t_{l+1}})
  ( \bar{\boldsymbol M}_{jj\sigma({t_{l+1}})}
  - \bar{\boldsymbol M}_{jj\sigma({t_{l+1}^{-}})} )
  s_{j}({t_{l+1}}) \nonumber\\
&\leq
\frac{{\bar{\varrho}_j} - {\ushort{\varrho}_j}}
{{2\ushort{\varrho}_j}}
s_{j}^{\top}({t_{l+1}})
\bar{\boldsymbol M}_{jj\sigma({t_{l+1}^{-}})}
s_{j}({t_{l+1}}) \nonumber\\
&\leq
\frac{{\bar{\varrho}_j} - {\ushort{\varrho}_j}}
{{\ushort{\varrho}_j}}
V_{j}(t_{l+1}^{-}) \nonumber \\
&\Rightarrow
V_{j}({t_{l+1}})
\leq \mu_j V_{j}(t_{l+1}^{-}),
\label{eq:mu_j}
\end{align}
with $\mu_j$ as defined in \eqref{mu}.
Hence, using (\ref{mu}) and (\ref{eq:mu_j}), from
(\ref{complete_lyap}) we have
\begin{align}
    V({t_{l+1}})
    &\leq \mu_p V_{{p}}(t_{l+1}^{-})
    + \mu_q V_{{q}}(t_{l+1}^{-})
    + \mu_\alpha V_{{\alpha}}(t_{l+1}^{-}) \nonumber \\
    &\leq \mu V(t_{l+1}^{-}).
    \label{eq:mu}
\end{align}

\subsection{Behaviour of $V(t)$ between consecutive switching instants}

For ease of analysis, we first determine $\dot{V}_p$, $\dot{V}_q$ and $\dot{V}_\alpha$ and then combine them to determine the overall closed-loop stability. For $j= p, q, \alpha$, the process is as follows: 

\subsubsection{\textbf{Analysis of $\dot{V}_j$}} 
\textbf{Scenario (i): $\norm{s_j} \geq  \varpi_j$}

\begin{align}
\dot{V}_j &= s_j^{\top}(t)
{\color{black}\bar{\boldsymbol M}_{jj\sigma}}
\dot{s}_j(t) \nonumber \\
&+ \sum_{\sigma=1}^{2}
\left\{
\sum_{i=0}^{3}
(\hat{K}_{ji\sigma}(t)-K_{ji\sigma}^{*})
\dot{\hat{K}}_{ji\sigma}(t)
+\frac{\dot{\zeta}_{j\sigma}(t)}{\ushort{\zeta}_j}
+\frac{\dot{\gamma}_{j\sigma}(t)}{\ushort{\gamma}_j}
\right\}\nonumber\\
&=
{\color{black}
s_j^{\top}
(-\boldsymbol\Lambda_{j\sigma}s_j
-\Delta\tau_{j\sigma}
+\mathcal{E}_{j\sigma})}
\nonumber\\
&+
\sum_{\sigma=1}^{2}
\left\{
\sum_{i=0}^{3}
(\hat{K}_{ji\sigma}(t)-K_{ji\sigma}^{*})
\dot{\hat{K}}_{ji\sigma}(t)
+\frac{\dot{\zeta}_{j\sigma}(t)}{\ushort{\zeta}_j}
+\frac{\dot{\gamma}_{j\sigma}(t)}{\ushort{\gamma}_j}
\right\}.
\nonumber
\end{align}

The adaptive law (\ref{adaptive_law_j}) reveals that $\gamma_{j\sigma}$ remains unchanged during active regimes, while $(\zeta_{j\bar{\sigma}}$, $\hat{K}_{ji\bar{\sigma}})$ remain unchanged during inactive regimes. Hence, we have
 \begin{align} \label{lyap_j_step2}
\dot{V}_j & \leq -s_j^{\top}{\boldsymbol \Lambda}_{j\sigma} s_j + 
{\color{black} \norm{\mathcal{E}_{j \sigma}} } \norm{s_j} - \rho_{j\sigma} \norm{s_j}  \nonumber \\ 
&+\sum_{i=0}^{3}(\hat{K}_{ji\sigma} - K_{ji\sigma}^{*}) \dot{\hat{K}}_{ji\sigma} + \frac{\dot{\zeta}_{j\sigma}}{\ushort{\zeta}_j} +\sum_{ \bar{\sigma} \in \mathcal{N}(\sigma)} \frac{\dot{\gamma}_{{j} \bar{\sigma}}} {\ushort{\gamma}_{j}} \nonumber \\
& \leq -s_j^{\top}{\boldsymbol \Lambda}_{j\sigma}s_j
- \sum_{i=0}^{2}(\hat{K}_{ji\sigma} - K_{ji\sigma}^{*})
\norm{\xi}^i \norm{s_j}
- {\color{black}\zeta_{j\sigma}} \norm{s_j}  \nonumber \\
&- ( \hat{K}_{j3\sigma} - K_{j3\sigma}^{*} )\norm{\ddot{\chi}}\norm{s_j} \nonumber \\ 
&+ \sum_{i=0}^{3}(\hat{K}_{ji\sigma} - K_{ji\sigma}^{*}) \dot{\hat{K}}_{ji\sigma} + \frac{\dot{\zeta}_{j\sigma}}{\ushort{\zeta}_j} +\sum_{ \bar{\sigma} \in \mathcal{N}(\sigma)} \frac{\dot{\gamma}_{{j} \bar{\sigma}}} {\ushort{\gamma}_{j}}, 
 \end{align} 
 where $\mathcal{N}(\sigma)$ denotes the set of inactive regimes corresponding to the active regime $\sigma$.

From \eqref{adaptive_law_j1} and \eqref{adaptive_law_j2} we have
 \begin{align} \label{lyap_j_step3}
 &\sum_{i=0}^{3}(\hat{K}_{ji\sigma} - K_{ji\sigma}^{*}) \dot{\hat{K}}_{ji\sigma} =  \nonumber \\
 &\sum_{i=0}^{2}(\hat{K}_{ji\sigma} - K_{ji\sigma}^{*}) (\norm{s_j} \norm{\xi}^i - \nu_{ji\sigma} \hat{K}_{ji\sigma})
 \nonumber \\ 
 & + 
 (\hat{K}_{j3\sigma} - K_{j3\sigma}^{*})(\norm{s_j} \norm{\ddot{\chi}} - \nu_{j3\sigma} \hat{K}_{j3\sigma})
 \nonumber \\ 
 &= \sum_{i=0}^{2}(\hat{K}_{ji\sigma} - K_{ji\sigma}^{*})\norm{\xi}^i \norm{s_j} + (\hat{K}_{j3\sigma} - K_{j3\sigma}^{*})\norm{\ddot{\chi}} \norm{s_j}  
 \nonumber \\ & + 
 \sum_{i=0}^{3}(\nu_{ji\sigma} \hat{K}_{ji\sigma} K_{ji\sigma}^{*} - \nu_{ji\sigma} \hat{K}_{ji\sigma}^2).  
 \end{align}
 
Further, one can verify that
 \begin{equation} \label{lyap_p_step4}
 \small
 (\nu_{ji\sigma} \hat{K}_{ji\sigma} K_{ji\sigma}^{*} - \nu_{ji\sigma} \hat{K}_{ji\sigma}^2)  \leq - \frac{\nu_{ji\sigma}}{2}  \left((\hat{K}_{ji\sigma} -  K_{ji\sigma}^*)^2 - {{K}^{*}}_{ji \sigma }^2 \right).
 \end{equation}

The adaptive law (\ref{adaptive_law_j4}) and relation (\ref{adaptive_gain_bound}) lead to
\begin{align} \label{lyap_j_step5}
 \frac{\dot{\gamma}_{{j} \bar{\sigma}}}{\ushort{\gamma}_j}=& -\left(1+ \frac{\varrho_{j\sigma}}{2}\sum \limits_{i=0}^{3} {\hat{K}_{j i\bar{\sigma}}}^2 \right) \frac{\gamma_{j \bar{\sigma}}}{\ushort{\gamma}_j} +\frac{\epsilon_{j \bar{\sigma}}}{\ushort{\gamma}_j} \nonumber \\
 &\leq -\frac{\varrho_{j\sigma}}{2} \sum \limits_{i=0}^{3} {\hat{K}_{j i\bar{\sigma}}}^2  +\frac{\epsilon_{j \bar{\sigma}}}{\ushort{\gamma}_j}.
\end{align}

 Substituting (\ref{lyap_j_step3})-(\ref{lyap_j_step5}) into (\ref{lyap_j_step2}) yields
 \begin{align} \label{lyap_j_step6}
\dot{V}_j(t)& \leq - \lambda_{\min}({\boldsymbol \Lambda}_{j\sigma})||s_j||^2  \nonumber \\
&-\sum_{i=0}^{3}\frac{\nu_{ji\sigma}}{2}  \left((\hat{K}_{ji\sigma} -  K_{ji\sigma}^*)^2 - {{K}^{*}}_{ji\sigma}^2 \right) \nonumber \\
& + \sum_{ \bar{\sigma} \in \mathcal{N}(\sigma)} \left ( \frac{ {\epsilon}_{j \bar{\sigma}}}{\ushort{\gamma}_{j}}  -\frac{\varrho_{j\sigma}}{2} \sum \limits_{i=0}^{3} {\hat{K}_{j i\bar{\sigma}}}^2\right).
 \end{align}

The definition of $V_j$ in (\ref{lyap_j}) yields 
\begin{align} \label{lyap_j_step7}
&V_j \leq \lambda_{\max} (\bar{\boldsymbol M}_{jj\sigma})\norm{s_j}^2   \nonumber \\ 
& + \sum_{i=0}^{3}\frac{(\hat{K}_{ji\sigma} - {K_{ji\sigma}^{*}})^2}{2} +\frac{\bar{\zeta}_{j\sigma}}{\ushort{\zeta}_j}  + \sum_{ \bar{\sigma} \in \mathcal{N}(\sigma)}   \frac{\bar{\gamma}_{j\bar{\sigma}}}{\ushort{\gamma}_j} .
 \end{align}
 
Using (\ref{lyap_j_step7}), the condition (\ref{lyap_j_step6}) is further simplified to 
\begin{align} \label{lyap_j_step8}
& \dot{V}_j \leq -\varrho_{j\sigma} V_j +  \frac{1}{2}\sum \limits_{i=0}^{3} \nu_{ji\sigma} {{K}^{*}}_{ji\sigma}^2  + \varrho_{j\sigma}  \frac{\bar{\zeta}_{j\sigma}}{\ushort{\zeta}_j}  \nonumber \\
&+\sum_{ \bar{\sigma} \in \mathcal{N}(\sigma)}\left( \frac{ {\epsilon}_{j \bar{\sigma}}}{\ushort{\gamma}_{j}} + \varrho_{j\sigma}  \frac{\bar{\gamma}_{j\bar{\sigma}}}{\ushort{\gamma}_j} \right).
 \end{align}

\textbf{Scenario (ii): $\norm{s_j} <  \varpi_j$}. In this case, we have
 \begin{align} \label{lyap_j_step9}
\dot{V}_j & \leq -s_j^{\top}{\boldsymbol \Lambda}_{j\sigma} s_j + \norm{\mathcal{E}_{j\sigma}} \norm{s_j} - \rho_{j\sigma} \frac{\norm{s_j}^2}{\varpi_j}  \nonumber \\ 
&+\sum_{i=0}^{3}(\hat{K}_{ji\sigma} - K_{ji\sigma}^{*}) \dot{\hat{K}}_{ji\sigma} + \frac{\dot{\zeta}_{j\sigma}}{\ushort{\zeta}_j} +\sum_{ \bar{\sigma} \in \mathcal{N}(\sigma)} \frac{\dot{\gamma}_{{j} \bar{\sigma}}} {\ushort{\gamma}_{j}} \nonumber \\
& \leq - \lambda_{\min}({\boldsymbol \Lambda}_{j\sigma})\norm{s_j}^2 + \norm{\mathcal{E}_{j\sigma}} \norm{s_j}  \nonumber \\ 
&+\sum_{i=0}^{3}(\hat{K}_{ji\sigma} - K_{ji\sigma}^{*}) \dot{\hat{K}}_{ji\sigma} + \frac{\dot{\zeta}_{j\sigma}}{\ushort{\zeta}_j} +\sum_{ \bar{\sigma} \in \mathcal{N}(\sigma)} \frac{\dot{\gamma}_{{j} \bar{\sigma}}} {\ushort{\gamma}_{j}} \nonumber \\
& {\small \leq - \lambda_{\min}
{\color{black}({\boldsymbol \Lambda}_{j\sigma}) } 
\norm{s_j}^2  +
\sum_{i=0}^{2}\hat{K}_{ji\sigma} \norm{\xi}^i \norm{s_j} + \hat{K}_{j3\sigma} \norm{\ddot{\chi}} \norm{s_j} } \nonumber \\
&-\sum_{i=0}^{3}\frac{\nu_{ji\sigma}}{2}  \left((\hat{K}_{ji\sigma} -  K_{ji\sigma}^*)^2 - {{K}^{*}}_{ji\sigma}^2 \right) + \frac{\dot{\zeta}_{j\sigma}}{\ushort{\zeta}_j} \nonumber \\
& + \sum_{ \bar{\sigma} \in \mathcal{N}(\sigma)} \left (\frac{ {\epsilon}_{j \bar{\sigma}}}{\ushort{\gamma}_{j}}  -\frac{\varrho_{j\sigma}}{2} \sum \limits_{i=0}^{3} {\hat{K}_{j i\bar{\sigma}}}^2\right).
 \end{align} 
The adaptive law (\ref{adaptive_law_j3}) and relation (\ref{adaptive_gain_bound}) lead to
 \begin{align} \footnotesize
 \label{lyap_j_step10}
\frac{\dot{\zeta}_{j\sigma}}{\ushort{\zeta}_j} \hspace{-1mm}  &= -\left(1 + (\hat{K}_{j3\sigma}\norm{\ddot{\chi}}+ \sum_{i=0}^{2}\hat{K}_{ji\sigma} \norm{\xi}^i )\norm{s_j} \right)  \frac{{\zeta_{j\sigma}}}{\ushort{\zeta}_j} + 
\hspace{-1mm} \frac{\delta_{j\sigma}}{\ushort{\zeta}_j}   \nonumber \\ 
& \leq  -\hat{K}_{j3\sigma} \norm{\ddot{\chi}} \norm{s_j} - \sum_{i=0}^{2}\hat{K}_{ji\sigma} \norm{\xi}^i \norm{s_j} + \frac{\delta_{j\sigma}}{\ushort{\zeta}_j}.
 \end{align}
Substituting (\ref{lyap_j_step10}) into (\ref{lyap_j_step9}) and using (\ref{lyap_j_step7}), $\dot{V}_j$ is simplified to 
\begin{align} \label{lyap_j_step11}
\dot{V}_j
&\leq
-{\color{black}\varrho_{j\sigma}}V_j
+\frac{1}{2}\sum \limits_{i=0}^{3}
{\color{black}\nu_{ji\sigma}{{K}^{*}}_{ji\sigma}^2}
+\frac{\delta_{j\sigma}}{\ushort{\zeta}_j}
+{\color{black}\varrho_{j\sigma}}
{\color{black}\frac{\bar{\zeta}_{j\sigma}}{\ushort{\zeta}_j}}
\nonumber\\
&{\color{black}
+\sum_{\bar{\sigma}\in\mathcal{N}(\sigma)}
\left(
\frac{{\epsilon}_{j\bar{\sigma}}}{\ushort{\gamma}_{j}}
+{\color{black}\varrho_{j\sigma}}
\frac{\bar{\gamma}_{j\bar{\sigma}}}{\ushort{\gamma}_j}
\right).}
\end{align}

\subsubsection{\textbf{Overall Stability Analysis}}
To obtain the overall behavior of $\dot{V}$, let us consider the following possible cases: \\
\textbf{Case (i): $\norm{s_p} \geq  \varpi_p, \norm{s_q} \geq  \varpi_q, \norm{s_\alpha} \geq  \varpi_\alpha$} \\
\textbf{Case (ii): $\norm{s_p} <  \varpi_p, \norm{s_q} <  \varpi_q, \norm{s_\alpha} <  \varpi_\alpha$} \\
\textbf{Case (iii): $\norm{s_p} <  \varpi_p, \norm{s_q} \geq  \varpi_q, \norm{s_\alpha} \geq  \varpi_\alpha$} \\
\textbf{Case (iv): $\norm{s_p} \geq  \varpi_p, \norm{s_q} <  \varpi_q, \norm{s_\alpha} \geq \varpi_\alpha$} \\
\textbf{Case (v): $\norm{s_p} \geq  \varpi_p, \norm{s_q} \geq  \varpi_q, \norm{s_\alpha} <  \varpi_\alpha$} \\
\textbf{Case (vi): $\norm{s_p} <  \varpi_p, \norm{s_q} <  \varpi_q, \norm{s_\alpha} \geq  \varpi_\alpha$} \\
\textbf{Case (vii): $\norm{s_p} \geq  \varpi_p, \norm{s_q} <  \varpi_q, \norm{s_\alpha} <  \varpi_\alpha$} \\
\textbf{Case (viii): $\norm{s_p} <  \varpi_p, \norm{s_q} \geq  \varpi_q, \norm{s_\alpha} <  \varpi_\alpha$}\\ 

Observing the results of $\dot{V}_j$, $j=p,q,\alpha$ under various scenarios as in (\ref{lyap_j_step8}), (\ref{lyap_j_step11}) the combined time derivative of ${V}$ for all Cases (i)-(viii) from (\ref{complete_lyap}) is obtained as  
\begin{align}
\dot{V} & \leq -\varrho V + \Delta,  \label{new}
\end{align}
\begin{multline}
\text{where}~\Delta =\sum \limits_{j=p, q, \alpha} \left( \frac{1}{2}\sum \limits_{i=0}^{3} \nu_{ji\sigma} {{K}^{*}}_{ji\sigma}^2     + \varrho_{j\sigma}  \frac{\bar{\zeta}_{j\sigma}}{\ushort{\zeta}_j} 
+\frac{\delta_{j\sigma}}{\ushort{\zeta}_j} \right. \\ 
\left. +\sum_{\bar{\sigma} \in \mathcal{N}(\sigma)}\left( \frac{{\epsilon}_{j \bar{\sigma}}}{\ushort{\gamma}_{j}} + \varrho_{j\sigma}  \frac{\bar{\gamma}_{j\bar{\sigma}}}{\ushort{\gamma}_j} \right) \right), \nonumber
\end{multline}
with $\varrho$ being defined in \eqref{mu_def}. Using the relation $0<\kappa<\varrho$ from \eqref{mu_def}, then (\ref{new}) can be written as
\begin{align} \label{V_dot_final}
\dot{V} & \leq  -\kappa V - (\varrho - \kappa)V + \Delta.
\end{align}
Defining a scalar $\mathcal{B} = \frac{\Delta}{(\varrho - \kappa)}$, it can be concluded that $\dot{V} (t) < - \kappa V (t)$ when $V (t) \geq \mathcal{ B}$. 
Next, further analysis is needed to observe the behaviour of $V(t)$ between the two consecutive switching instants, i.e., $t \in [t_{l}~t_{l+1})$, for two possible cases: 
\begin{itemize}
\item when $V(t) \geq \mathcal{B}$, we have $\dot{V}(t) \leq - \kappa V(t) $ implying exponential decrease of $V(t)$;
\item when $V(t) <\mathcal{B}$, no exponential decrease can be derived.
\end{itemize}
The behavior of $V(t)$ is discussed below individually for these two cases.

When $V(t) \geq \mathcal{B}$,
there exists a time, call it $\iota$, when $V(t)$ enters into the bound $\mathcal{B}$ and $N_\sigma(t)$ denotes the number of all switching intervals for $t \in [0 ~~\iota)$. Accordingly, for $t \in [0 ~~\iota)$, using (\ref{eq:mu}), (\ref{V_dot_final}) and $N_{\sigma}(0,t)$ from Definition \ref{def1} we have 
\begin{align}
  V(t)   \leq & \exp \left( - \kappa (t-t_{ N_\sigma(t)-1})\right) V(t_{N_\sigma(t)-1}) \nonumber\\
\leq & \mu\exp \left( - \kappa (t-t_{N_\sigma(t)-1})\right) V(t_{N_\sigma(t)-1}^{-}) \nonumber\\
\leq & \mu\exp \left( - \kappa (t-t_{N_\sigma(t)-1})\right) \cdot \mu \nonumber \\
& \qquad \exp \left( - \kappa (t_{N_\sigma(t)-1}-t_{N_\sigma(t)-2}) \right)V(t_{N_\sigma(t)-2}^{-}) \nonumber\\
&\qquad \qquad \qquad \qquad \vdots \nonumber\\
\leq & \mu\exp \left( - \kappa (t-t_{N_\sigma(t)-1})\right) \cdot \mu \nonumber \\
& \hspace{-1mm} \exp  \left(  - \kappa (t_{N_\sigma(t)-1}-t_{N_\sigma(t)-2}) \right)  \cdots \hspace{-.5mm} \mu \exp \left(-\kappa (t_1) \right) V(0) \nonumber\\
= & \mu^{N_{\sigma}(0,t)}
\exp \left(- \kappa t \right)V(0)  \nonumber \\
{\color{black}\leq} &
c \left(
\exp \left(
-\kappa + ({\ln \mu}/{\vartheta})
\right)t
\right)V(0),
\label{eq:dot_V_part_8}
\end{align}
where $c \triangleq \exp \left( N_0 \ln \mu \right) $ is a constant. It can be noted from (\ref{eq:dot_V_part_8}) that for achieving stability $V(t) < c V(0)$, one needs to satisfy the ADT condition $\vartheta > \ln\mu / \kappa $ as in (\ref{eq:sw_law}). Moreover, as $V(\iota) < \mathcal{B}$, one has $V(t_{\bar N(t)+1}) < \mu \mathcal{B}$ from (\ref{eq:mu}) at the next switching instant $t_{\bar N(t)+1}$ after $\iota$. 
By following the standard lines of recursive arguments (\hspace{-1mm} \cite{yuan2017robust, hespanha1999stability}), one can conclude that $V(t) < c \mu \mathcal{B} $ for $t \in [t_0+\iota~~ \infty )$. This confirms that once $V(t)$ enters the interval $[0~\mathcal{B}]$, it cannot exceed the bound $c\mu \mathcal{B}$ any time later with the ADT switching law (\ref{eq:sw_law}).

It can be easily verified that the same argument below (\ref{eq:dot_V_part_8}) also holds for when $V(t) < \mathcal{B}$. 

It can be concluded from the stability arguments of the above two cases that the closed-loop system remains UUB globally leading to
\begin{align}
V(t) \leq \max \left \lbrace c V(0), c\mu \mathcal{B} \right \rbrace, ~~\forall t\geq 0. \label{eq:ub_1}
\end{align}
Again, the definition of $V(t)$ in (\ref{complete_lyap}) yields
\begin{align}
{\color{black}
V(t)\geq V_j(t)\geq
\frac{1}{2}\mathcal{M}_j\norm{s_j}^2,
}
~{j=p,q,\alpha}\label{eq:ub_2}
\end{align}
where $\mathcal{M}_j =
\min_{\sigma \in \Omega}
\{\lambda_{\min}(\bar{\boldsymbol M}_{jj\sigma})\}$.
Using (\ref{eq:ub_1}) and (\ref{eq:ub_2}) we obtain a tracking error bound
on individual quadrotor position, attitude and manipulator sub-dynamics as
\begin{align}
\norm{s_j}
\leq
\sqrt{
\frac{{\color{black}2}}{\mathcal{M}_j}
\max \left\lbrace
cV(0),c\mu\mathcal{B}
\right\rbrace
},
~~\forall t\geq0
\label{eq:ub_3}
\end{align}
which can be tuned via user-defined parameters from \eqref{ADT_gains}. 

\begin{remark}[Role of $\zeta_{j\sigma}$ and $\gamma_{j\sigma}$] \label{remark_zeta_gamma}
Note that the boundedness of $||s_j|| \leq \varpi_j, j = p, q, \alpha$ in Scenario (ii) of $\dot{V}_j$ ensures the boundedness of   $||\xi_j||$  but does not guarantee the boundedness of $||\ddot{\chi}||$ or $||\xi||$. Therefore, to ensure closed-loop stability, it is necessary to eliminate the terms `($\hat{K}_{j3\sigma} ||\ddot{\chi}|| ||s_j|| + \sum_{i=0}^{2}\hat{K}_{ji\sigma} ||\xi||^i ||s_j||) $', which is possible using the auxiliary gains $\zeta_{j\sigma}$ (cf. $\dot{\zeta}_{j\sigma}$).
On the other hand, the auxiliary gain $\gamma_{j\sigma}$ ensures stability across switching instances by compensating for sudden variations in system dynamics, preventing excessive transients in adaptation. Hence, $\zeta_{j\sigma}$ maintains Lyapunov boundedness, while $\gamma_{j\sigma}$ stabilizes regime transitions.
\end{remark}

\vspace{-5pt}

\bibliographystyle{IEEEtran}
\bibliography{root}

@article{hanover2021performance,
  title={Performance, precision, and payloads: Adaptive nonlinear {MPC} for quadrotors},
  author={Hanover, Drew and Foehn, Philipp and Sun, Sihao and Kaufmann, Elia and Scaramuzza, Davide},
  journal={IEEE Robotics and Automation Letters},
  volume={7},
  number={2},
  pages={690--697},
  year={2021},
  publisher={IEEE}
}

@inproceedings{arleo2013control,
  title={Control of quadrotor aerial vehicles equipped with a robotic arm},
  author={Arleo, G and Caccavale, Fabrizio and Muscio, Giuseppe and Pierri, Francesco},
  booktitle={21St Mediterranean Conference on Control and Automation},
  pages={1174--1180},
  year={2013},
  organization={IEEE}
}

@book{liberzon2003switching,
  title={Switching in Systems and Control},
  author={Liberzon, Daniel},
  year={2003},
  publisher={Springer Science \& Business Media}
}

@article{hespanha1999stability,
  title={Stability of switched systems with average dwell-time},
  author={Hespanha, Joao P and Morse, A Stephen},
  journal={Proceedings of the 38th IEEE Conference on Decision and Control},
  pages={2655--2660},
  year={1999},
}

@book{spong2008robot,
  title={Robot Dynamics and Control},
  author={Spong, Mark W and Vidyasagar, Mathukumalli},
  year={2008},
  publisher={John Wiley \& Sons}
}

@online{mav_trajectory_generation,
  author = {{Autonomous Systems Lab {ETH Zurich}}},
  title = {{MAV} Trajectory Generation},
  year = {},
  url = {https://github.com/ethz-asl/mav_trajectory_generation},
  urldate = {}
}

@article{berscheid2021jerk,
  title={Jerk-limited Real-time Trajectory Generation with Arbitrary Target States},
  author={Berscheid, Lars and Kr{\"o}ger, Torsten},
  journal={Robotics: Science and Systems XVII},
  year={2021},
url = {https://github.com/pantor/ruckig},
}

@inproceedings{mellinger2011minimum,
  title={Minimum snap trajectory generation and control for quadrotors},
  author={Mellinger, Daniel and Kumar, Vijay},
  booktitle={2011 IEEE International Conference on Robotics and Automation (ICRA)},
  pages={2520--2525},
  year={2011},
  organization={IEEE}
}

@inproceedings{yuan2017robust,
  title={On robust adaptive control of switched linear systems},
  author={Yuan, Shuai and De Schutter, Bart and Baldi, Simone},
  booktitle={2017 IEEE International Conference on Control \& Automation (ICCA)},
  pages={753--758},
  year={2017},
  organization={IEEE}
}

@article{kim2016vision,
  title={Vision-guided aerial manipulation using a multirotor with a robotic arm},
  author={Kim, Suseong and Seo, Hoseong and Choi, Seungwon and Kim, H Jin},
  journal={IEEE/ASME Transactions on Mechatronics},
  volume={21},
  number={4},
  pages={1912--1923},
  year={2016},
  publisher={IEEE}
}

@book{orsag2018aerial,
  title={Aerial Manipulation},
  author={Orsag, Matko and Korpela, Christopher and Oh, Paul and Bogdan, Stjepan and Ollero, Anibal},
  year={2018},
  publisher={Springer}
}

@article{ruggiero2018aerial,
  title={Aerial manipulation: A literature review},
  author={Ruggiero, Fabio and Lippiello, Vincenzo and Ollero, Anibal},
  journal={IEEE Robotics and Automation Letters},
  volume={3},
  number={3},
  pages={1957--1964},
  year={2018},
  publisher={IEEE}
}

@article{meng2020survey,
  title={Survey on aerial manipulator: System, modeling, and control},
  author={Meng, Xiangdong and He, Yuqing and Han, Jianda},
  journal={Robotica},
  volume={38},
  number={7},
  pages={1288--1317},
  year={2020},
  publisher={Cambridge University Press}
}

@article{kim2017robust,
  title={Robust control of an equipment-added multirotor using disturbance observer},
  author={Kim, Suseong and Choi, Seungwon and Kim, Hyeonggeun and Shin, Jongho and Shim, Hyungbo and Kim, H Jin},
  journal={IEEE Transactions on Control Systems Technology},
  volume={26},
  number={4},
  pages={1524--1531},
  year={2017},
  publisher={IEEE}
}

@article{chen2020robust,
  title={Robust control for unmanned aerial manipulator under disturbances},
  author={Chen, Yanjie and Zhan, Weiwei and He, Bingwei and Lin, Lixiong and Miao, Zhiqiang and Yuan, Xiaofang and Wang, Yaonan},
  journal={IEEE Access},
  volume={8},
  pages={129869--129877},
  year={2020},
  publisher={IEEE}
}

@inproceedings{ruggiero2015multilayer,
  title={A multilayer control for multirotor {UAV}s equipped with a servo robot arm},
  author={Ruggiero, Fabio and Trujillo, Miguel Angel and Cano, Raul and Ascorbe, H and Viguria, Antidio and Per{\'e}z, C and Lippiello, Vincenzo and Ollero, An{\'\i}bal and Siciliano, Bruno},
  booktitle={2015 IEEE International Conference on Robotics and Automation (ICRA)},
  pages={4014--4020},
  year={2015},
  organization={IEEE}
}

@article{chen2022adaptive,
  title={Adaptive Sliding-Mode Disturbance Observer-Based Finite-Time Control for Unmanned Aerial Manipulator With Prescribed Performance},
  author={Chen, Yanjie and Liang, Jiacheng and Wu, Yangning and Miao, Zhiqiang and Zhang, Hui and Wang, Yaonan},
  journal={IEEE Transactions on Cybernetics},
volume={53},
  number={5},
  pages={3263--3276},
  year={2023},
  publisher={IEEE}
}

@article{zhang2019robust,
  title={Robust control of an aerial manipulator based on a variable inertia parameters model},
  author={Zhang, Guangyu and He, Yuqing and Dai, Bo and Gu, Feng and Han, Jianda and Liu, Guangjun},
  journal={IEEE Transactions on Industrial Electronics},
  volume={67},
  number={11},
  pages={9515--9525},
  year={2019},
  publisher={IEEE}
}

@article{liang2021low,
  title={Low-complexity prescribed performance control for unmanned aerial manipulator robot system under model uncertainty and unknown disturbances},
  author={Liang, Jiacheng and Chen, Yanjie and Lai, Ningbin and He, Bingwei and Miao, Zhiqiang and Wang, Yaonan},
  journal={IEEE Transactions on Industrial Informatics},
  volume={18},
  number={7},
  pages={4632--4641},
  year={2021},
  publisher={IEEE}
}

@article{orsag2017dexterous,
  title={Dexterous aerial robots—mobile manipulation using unmanned aerial systems},
  author={Orsag, Matko and Korpela, Christopher and Bogdan, Stjepan and Oh, Paul},
  journal={IEEE Transactions on Robotics},
  volume={33},
  number={6},
  pages={1453--1466},
  year={2017},
  publisher={IEEE}
}

@article{tognon2019truly,
  title={A truly-redundant aerial manipulator system with application to push-and-slide inspection in industrial plants},
  author={Tognon, Marco and Ch{\'a}vez, Hermes A Tello and Gasparin, Enrico and Sabl{\'e}, Quentin and Bicego, Davide and Mallet, Anthony and Lany, Marc and Santi, Gilles and Revaz, Bernard and Cort{\'e}s, Juan and others},
  journal={IEEE Robotics and Automation Letters},
  volume={4},
  number={2},
  pages={1846--1851},
  year={2019},
  publisher={IEEE}
}

@article{lee2022rise,
  title={{RISE}-based trajectory tracking control of an aerial manipulator under uncertainty},
  author={Lee, Dongjae and Byun, Jeonghyun and Kim, H Jin},
  journal={IEEE Control Systems Letters},
volume={6},
  pages={3379--3384},
  year={2022},
  publisher={IEEE}
}

@article{liang2022adaptive,
  title={Adaptive Prescribed Performance Control of Unmanned Aerial Manipulator With Disturbances},
  author={Liang, Jiacheng and Chen, Yanjie and Wu, Yangning and Miao, Zhiqiang and Zhang, Hui and Wang, Yaonan},
  journal={IEEE Transactions on Automation Science and Engineering},
volume={20},
  number={3},
  pages={1804--1814},
  year={2023},
  publisher={IEEE}
}

@inproceedings{jimenez2013control,
  title={Control of an aerial robot with multi-link arm for assembly tasks},
  author={Jimenez-Cano, Antonio E and Martin, Jes{\'u}s and Heredia, Guillermo and Ollero, An{\'\i}bal and Cano, Raul},
  booktitle={2013 IEEE International Conference on Robotics and Automation (ICRA)},
  pages={4916--4921},
  year={2013},
  organization={IEEE}
}

@article{suarez2020benchmarks,
  title={Benchmarks for aerial manipulation},
  author={Suarez, Alejandro and Vega, Victor M and Fernandez, Manuel and Heredia, Guillermo and Ollero, Anibal},
  journal={IEEE Robotics and Automation Letters},
  volume={5},
  number={2},
  pages={2650--2657},
  year={2020},
  publisher={IEEE}
}

@inproceedings{tilli2021low,
  title={Low-Input Accurate Periodic Motion of an Underactuated Mechanism: Mass Distribution and Nonlinear Spring Shaping},
  author={Tilli, Andrea and Ruggiano, Elena and Bosso, Alessandro and Samor{\`\i}, Alessandro},
  booktitle={2021 IEEE/ASME International Conference on Advanced Intelligent Mechatronics (AIM)},
  pages={292--299},
  year={2021},
  organization={IEEE}
}

@article{cao2019faster,
  title={Faster fixed-time control of flexible spacecraft attitude stabilization},
  author={Cao, Lu and Xiao, Bing and Golestani, Mehdi and Ran, Dechao},
  journal={IEEE Transactions on Industrial Informatics},
  volume={16},
  number={2},
  pages={1281--1290},
  year={2020},
  publisher={IEEE}
}

@article{xiao2021prescribed,
  title={Prescribed time attitude tracking control of spacecraft with arbitrary disturbance},
  author={Xiao, Bing and Wu, Xiwei and Cao, Lu and Hu, Xiaoxiang},
  journal={IEEE Transactions on Aerospace and Electronic Systems},
  volume={58},
  number={3},
  pages={2531--2540},
  year={2022},
  publisher={IEEE}
}

@article{lai2018adaptive,
  title={Adaptive backstepping-based tracking control of a class of uncertain switched nonlinear systems},
  author={Lai, Guanyu and Liu, Zhi and Zhang, Yun and Chen, CL Philip and Xie, Shengli},
  journal={Automatica},
  volume={91},
  pages={301--310},
  year={2018},
  publisher={Elsevier}
}

@article{yuan2018robust,
  title={Robust adaptive tracking control of uncertain slowly switched linear systems},
  author={Yuan, Shuai and De Schutter, Bart and Baldi, Simone},
  journal={Nonlinear Analysis: Hybrid Systems},
  volume={27},
  pages={1--12},
  year={2018},
  publisher={Elsevier}
}

@article{yuan2016adaptive,
  title={Adaptive asymptotic tracking control of uncertain time-driven switched linear systems},
  author={S. {Yuan} and De Schutter, Bart and Baldi, Simone},
  journal={IEEE Transactions on Automatic Control},
  volume={62},
  number={11},
  pages={5802--5807},
  year={2016},
  publisher={IEEE}
}

@article{lou2018immersion,
  title={Immersion-and invariance-based adaptive stabilization of switched nonlinear systems},
  author={Lou, Zhi-E and Zhao, Jun},
  journal={International Journal of Robust and Nonlinear Control},
  volume={28},
  number={1},
  pages={197--212},
  year={2018},
  publisher={Wiley Online Library}
}

@article{8362915,
title={Switched Adaptive Control of Air Handling Units With Discrete and Saturated Actuators},
author={S. {Yuan} and L. {Zhang} and O. {Holub} and S. {Baldi}},
journal={IEEE Control Systems Letters},
volume={2},
number={3},
pages={417-422},
year={2018},
publisher={IEEE}
}

@article{roy2019reduced,
  title={On reduced-complexity robust adaptive control of switched {Euler-Lagrange} systems},
  author={Roy, Spandan and Baldi, Simone},
  journal={Nonlinear Analysis: Hybrid Systems},
  volume={34},
  pages={226--237},
  year={2019},
  publisher={Elsevier}
}

@article{roy2019simultaneous,
  title={A Simultaneous Adaptation Law for a Class of Nonlinearly Parametrized Switched Systems},
  author={Roy, Spandan and S. {Baldi}},
  journal={IEEE Control Systems Letters},
  volume={3},
  number={3},
  pages={487--492},
  year={2019},
  publisher={IEEE}
}

@ARTICLE{9813358,
  author={Rashad, Ramy and Bicego, Davide and Zult, Jelle and Sanchez-Escalonilla, Santiago and Jiao, Ran and Franchi, Antonio and Stramigioli, Stefano},
  journal={IEEE Transactions on Robotics}, 
  title={Energy Aware Impedance Control of a Flying End-Effector in the Port-{Hamiltonian} Framework}, 
  year={2022},
  volume={38},
  number={6},
  pages={3936-3955},
  doi={10.1109/TRO.2022.3183532}}

@article{cao2023eso,
  title={{ESO}-based robust and high-precision tracking control for aerial manipulation},
  author={Cao, Huazi and Li, Yongqi and Liu, Cunjia and Zhao, Shiyu},
  journal={IEEE Transactions on Automation Science and Engineering},
  volume={21},
  number={2},
  pages={2139--2155},
  year={2023},
  publisher={IEEE}
}

@ARTICLE{10339889,
  author={Wang, Meng and Chen, Zeshuai and Guo, Kexin and Yu, Xiang and Zhang, Youmin and Guo, Lei and Wang, Wei},
  journal={IEEE Transactions on Robotics}, 
  title={Millimeter-Level Pick and Peg-in-Hole Task Achieved by Aerial Manipulator}, 
  year={2024},
  volume={40},
  number={},
  pages={1242-1260},
  doi={10.1109/TRO.2023.3338956}}

@ARTICLE{10758214,
  author={Li, Guanrui and Liu, Xinyang and Loianno, Giuseppe},
  journal={IEEE Transactions on Robotics}, 
  title={Human-Aware Physical Human–Robot Collaborative Transportation and Manipulation With Multiple Aerial Robots}, 
  year={2025},
  volume={41},
  number={},
  pages={762-781},
  doi={10.1109/TRO.2024.3502508}}

@ARTICLE{9462539,
  author={Ollero, Anibal and Tognon, Marco and Suarez, Alejandro and Lee, Dongjun and Franchi, Antonio},
  journal={IEEE Transactions on Robotics}, 
  title={Past, Present, and Future of Aerial Robotic Manipulators}, 
  year={2022},
  volume={38},
  number={1},
  pages={626-645},
  doi={10.1109/TRO.2021.3084395}}

@ARTICLE{10466505,
  author={Liang, Xiao and Wang, Yang and Yu, Hai and Zhang, Zhaopeng and Han, Jianda and Fang, Yongchun},
  journal={IEEE Transactions on Automation Science and Engineering}, 
  title={Observer-Based Nonlinear Control for Dual-Arm Aerial Manipulator Systems Suffering From Uncertain Center of Mass}, 
  year={2025},
  volume={22},
  number={},
  pages={1984-1995},
  doi={10.1109/TASE.2024.3373107}}

@ARTICLE{10505853,
  author={Li, Hai and Li, Zhan and Song, Fulin and Yu, Xinghu and Yang, Xuebo and Rodríguez-Andina, Juan J.},
  journal={IEEE Transactions on Industrial Electronics}, 
  title={Finite-Time Fast Adaptive Backstepping Attitude Control for Aerial Manipulators Based on Variable Coupling Disturbance Compensation}, 
  year={2024},
  volume={71},
  number={11},
  pages={14730-14739},
  doi={10.1109/TIE.2024.3383036}}

@ARTICLE{10722859,
  author={Li, Zhan and Li, Hai and Xu, Quman and Yu, Xinghu and Basin, Michael V.},
  journal={IEEE Transactions on Cybernetics}, 
  title={Coupling Disturbance Modeling and Compensation for Aerial Manipulator in Highly Dynamic Motion}, 
  year={2025},
  volume={55},
  number={1},
  pages={124-135},
  doi={10.1109/TCYB.2024.3476124}}

@ARTICLE{10701509,
  author={Yadav, Rishabh Dev and Dantu, Swati and Pan, Wei and Sun, Sihao and Roy, Spandan and Baldi, Simone},
  journal={IEEE/ASME Transactions on Mechatronics}, 
  title={Modular Adaptive Aerial Manipulation Under Unknown Dynamic Coupling Forces}, 
  year={2024},
  volume={},
  number={},
  pages={1-11},
  doi={10.1109/TMECH.2024.3457806}}

@ARTICLE{8249868,
  author={Chen, Weisheng and Wen, Changyun and Wu, Jian},
  journal={IEEE Transactions on Automatic Control}, 
  title={Global Exponential/Finite-Time Stability of Nonlinear Adaptive Switching Systems With Applications in Controlling Systems With Unknown Control Direction}, 
  year={2018},
  volume={63},
  number={8},
  pages={2738-2744},
  doi={10.1109/TAC.2018.2791346}}

@inproceedings{singh2026aerograb,
  author    = {Singh, Shivansh Pratap and Nair, Naveen Sudheer and Ujjawal, Samaksh and Mishra, Sarthak and Patil, Soham and Yadav, Rishabh Dev and Roy, Spandan},
  title     = {{AeroGrab}: A Unified Framework for Aerial Grasping in Cluttered Environments},
  booktitle = {2026 IEEE/RSJ International Conference on Intelligent Robots and Systems (IROS)},
  year      = {2026},
  note      = {Accepted},
  eprint    = {2603.15097},
  archivePrefix = {arXiv},
  primaryClass  = {cs.RO}
}

@inproceedings{mishra2026aeroplace,
  author    = {Mishra, Sarthak and Yadav, Rishabh Dev and Nair, Naveen and Pan, Wei and Roy, Spandan},
  title     = {{AeroPlace-Flow}: Language-Grounded Object Placement for Aerial Manipulators via Visual Foresight and Object Flow},
  booktitle = {2026 IEEE/RSJ International Conference on Intelligent Robots and Systems (IROS)},
  year      = {2026},
  note      = {Accepted},
  eprint    = {2603.07744},
  archivePrefix = {arXiv},
  primaryClass  = {cs.RO}
}

@inproceedings{ujjawal2026learn,
  author    = {Ujjawal, Samaksh and Nair, Naveen Sudheer and Singh, Shivansh Pratap and Yadav, Rishabh Dev and Pan, Wei and Roy, Spandan},
  title     = {Learn Structure, Adapt on the Fly: Multi-Scale Residual Learning and Online Adaptation for Aerial Manipulators},
  booktitle = {2026 IEEE/RSJ International Conference on Intelligent Robots and Systems (IROS)},
  year      = {2026},
  note      = {Accepted},
  eprint    = {2603.11638},
  archivePrefix = {arXiv},
  primaryClass  = {cs.RO}
}

@inproceedings{ujjawal2026aermanidiffusion,
  author    = {Ujjawal, Samaksh and Singh, Shivansh Pratap and Nair, Naveen Sudheer and Yadav, Rishabh Dev and Pan, Wei and Roy, Spandan},
  title     = {{AERMANI-Diffusion}: Regime-Conditioned Diffusion for Dynamics Learning in Aerial Manipulators},
  booktitle = {Proceedings of the 23rd IFAC World Congress},
  year      = {2026},
  address   = {Busan, Republic of Korea},
  note      = {Accepted},
  eprint    = {2512.10773},
  archivePrefix = {arXiv},
  primaryClass  = {cs.RO}
}

@ARTICLE{yadav2025integrated,
  author={Yadav, Rishabh Dev and Jones, Brycen and Gupta, Saksham and Sharma, Amitabh and Sun, Jiefeng and Zhao, Jianguo and Roy, Spandan},
  journal={IEEE/ASME Transactions on Mechatronics}, 
  title={An Integrated Approach to Aerial Grasping: Combining a Bistable Gripper With Adaptive Control}, 
  year={2026},
  volume={31},
  number={1},
  pages={245-256},
  doi={10.1109/TMECH.2025.3586888}}

\begin{IEEEbiography}[{\includegraphics[width=1in,height=1.25in,clip,keepaspectratio]{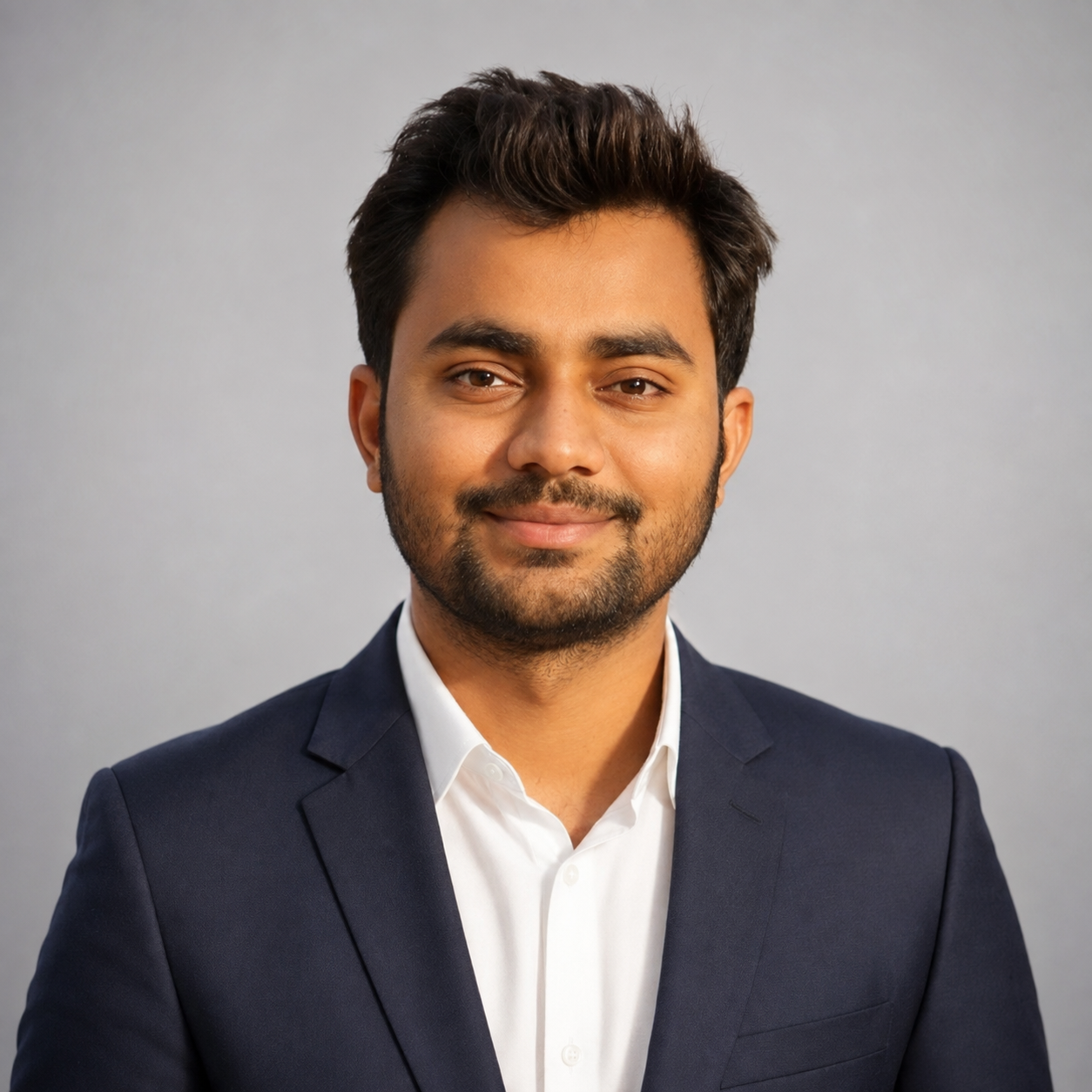}}]{Rishabh Dev Yadav}
received the B.Tech. degree in Mechanical Engineering from Indian Institute of Information Technology, Jabalpur, India, in 2019, and the M.S. by Research degree in Computer Science and Engineering from International Institute of Information Technology Hyderabad, India, in 2023. He is currently a Ph.D. student in Computer Science at The University of Manchester, UK, and an ELLIS Ph.D. student. He was a Visiting Ph.D. Researcher at TU Delft, The Netherlands. His research interests lie at the intersection of machine learning, control, and aerial robotics, with a focus on applied adaptive-robust control, data-driven control, system identification, dynamics learning, and aerial manipulation.
\end{IEEEbiography}

\begin{IEEEbiography}[{\includegraphics[width=1in,height=1.25in,clip,keepaspectratio]{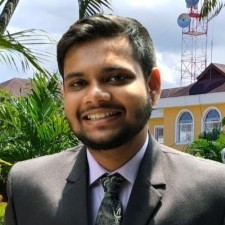}}]{Saksham Gupta} received his B.Tech in Mechatronics from Mukesh Patel School of Technology Management and Engineering (NMIMS Univerity, Mumbai) in 2020 and is currently pursuing Master of Science in Electronic and Communication Engineering by Research from International Institute of Information Technology Hyderabad.
His current research interests include applied adaptive-robust control and Bio-Medical Robotics.
\end{IEEEbiography}

\begin{IEEEbiography}[{\includegraphics[width=1in,height=1.25in,clip,keepaspectratio]{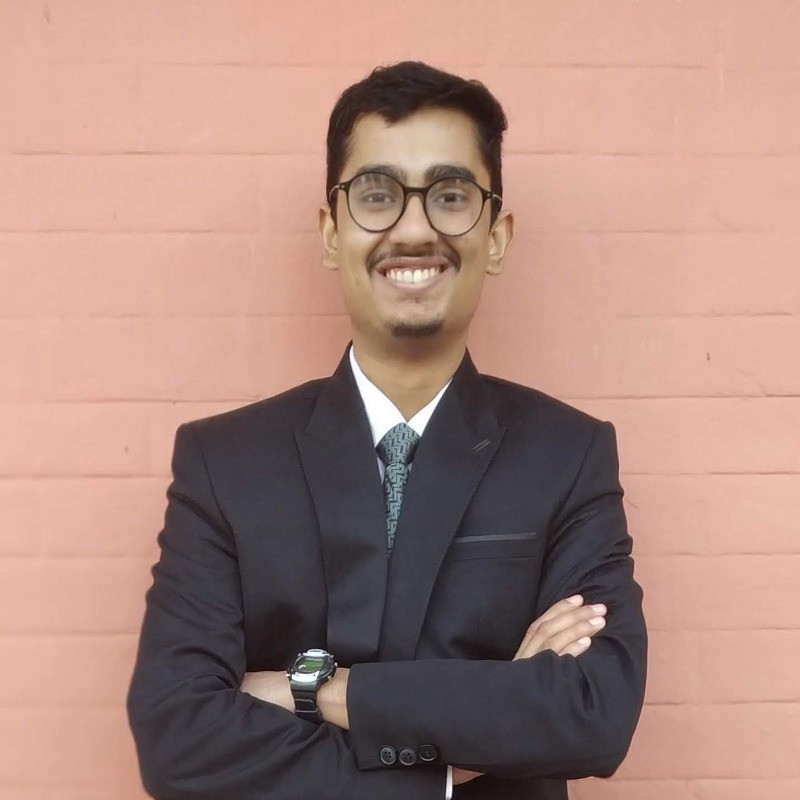}}]{Amitabh Sharma} received his Bachelor of Engineering from Birla Institute of Technology and Science, Pilani, in Electronics and Instrumentation Engineering in 2020 and is currently pursuing Master of Science in Electronic and Communication Engineering by Research from International Institute of Information Technology Hyderabad.
His current research interests include adaptive-robust control applied to various robotic systems.
\end{IEEEbiography}

\begin{IEEEbiography}
[{\includegraphics[width=1in,height=1.25in,clip,keepaspectratio]{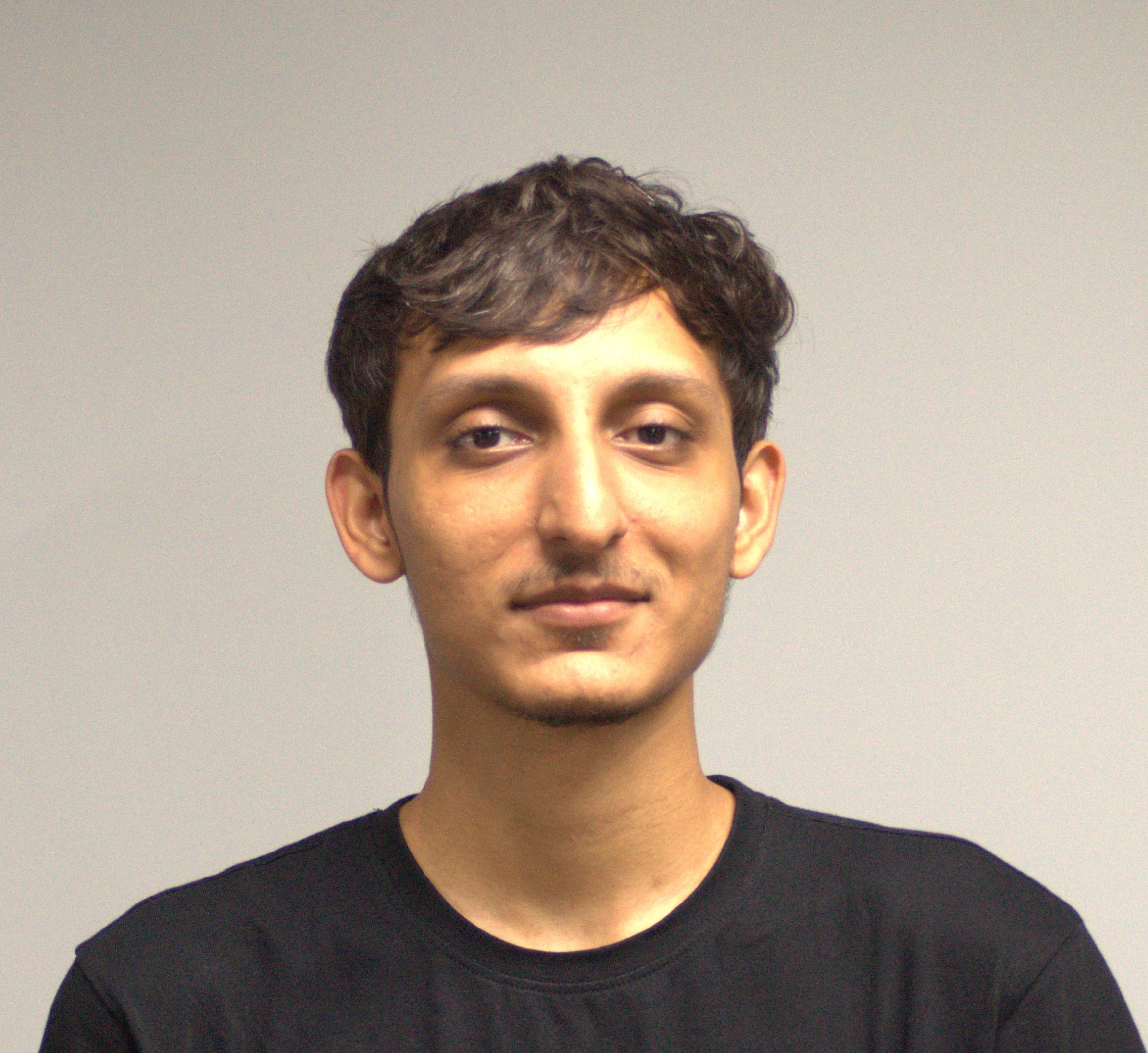}}]{Sarthak Mishra} received his B.Tech degree in Computer Engineering from NMIMS's Mukesh Patel School of Technology Management and Engineering Mumbai, India in 2023. He is currently pursuing Master's of Science in Computer Science and Engineering by Research at the International Institute of Information Technology Hyderabad, India. His research interests include language guided control and adaptive control. 
	\end{IEEEbiography}

\begin{IEEEbiography}[{\includegraphics[width=1in,height=1.25in,clip,keepaspectratio]{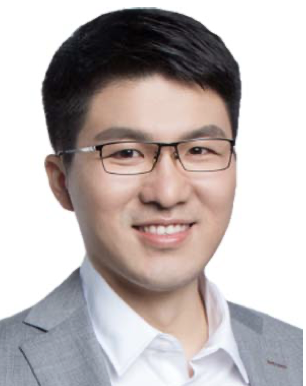}}]{Wei Pan} received the Ph.D. degree in bioengineering from Imperial College London, London, U.K., in 2016. He is currently a Full Professor of Autonomous Systems and Automatic Control with the School of Engineering, Newcastle University, Newcastle upon Tyne, U.K. His research interests include (any and multi) robot control using Bayesian machine learning and the principles of dynamic control. He is a Senior Editor of IEEE Robotics and Automation Letters and ACM Transactions on Probabilistic Machine Learning and an Associate Editor or the Area Chair of IEEE Transactions on Robotics, IEEE Transactions on Robot Learning, IEEE Transactions on Automation Science and Engineering, the International Conference on Robotics and Automation, the International Conference on Intelligent Robots and Systems, the Robotics: Science and Systems, the Conference on Robot Learning, the Conference on Neural Information Processing Systems, and the Learning for Dynamics and Control Conference.
\end{IEEEbiography}

\begin{IEEEbiography}[{\includegraphics[width=1in,height=1.25in,clip,keepaspectratio]{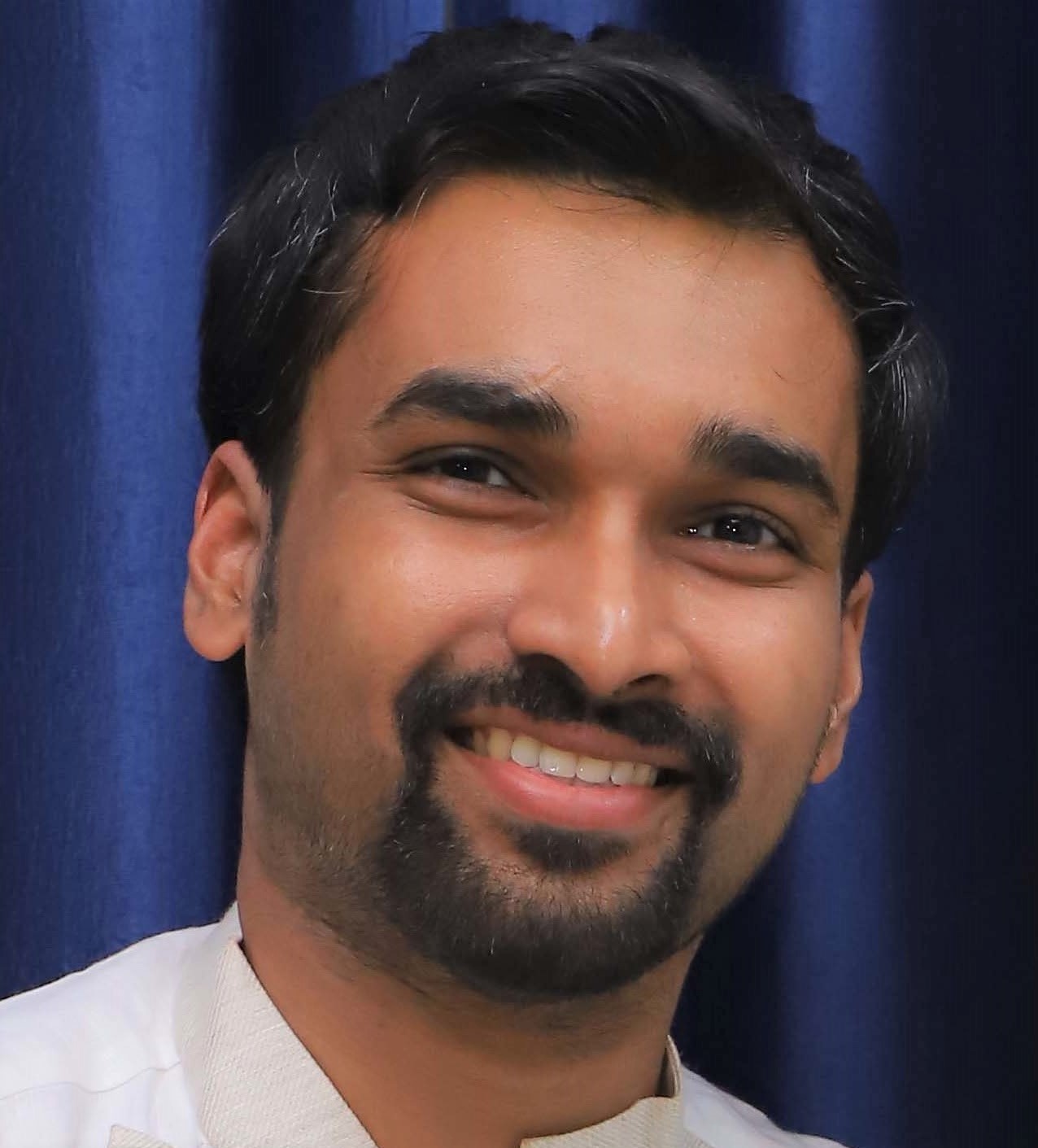}}]{Spandan Roy} (M'18) received his B.Tech degree in Electronics and Communication Engineering from Techno India, West Bengal University of Technology,  India in 2011, M.Tech. degree in Mechatronics from Academy of Scientific and Innovative Research, India in 2013 and Ph.D. degree in Control and Automation from Indian Institute of Technology Delhi, India in 2018. He is currently an assistant professor at the Robotics Research Center, International Institute of Information Technology Hyderabad, India. Previously, he was a postdoctoral researcher in Delft Center for System and Control, TU Delft. He is a subject editor of \emph{Int. Journal of Adaptive Control and Signal Processing}, an associate editor of \emph{IEEE Control Systems Letters} and a technical editor of \emph{IEEE/ASME Trans. on Mechatronics}. His research interests include adaptive-robust control, switched systems, and its applications in Euler-Lagrange systems.
\end{IEEEbiography}

\begin{IEEEbiography}[{\includegraphics[width=1in,height=1.25in,clip,keepaspectratio]{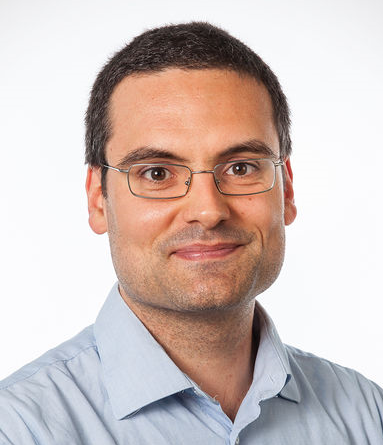}}]
{Simone Baldi} (M'14, SM'19) received the B.Sc. in electrical engineering, and the M.Sc. and Ph.D. in automatic control engineering from University of Florence, Italy, in 2005, 2007, and 2011. Since 2019, he is a Professor with the School of Mathematics, Southeast University. Before that, he was an assistant professor at Delft Center for Systems and Control, TU Delft from 2014 to 2019. He was awarded outstanding reviewer for \emph{Applied Energy} (2016) and \emph{Automatica} (2017) and outstanding associate editor
for \emph{Journal of the Franklin Institute} (2023-2024) and \emph{IEEE Control Systems Letters} (2024). He is a subject editor of \emph{Int. Journal of Adaptive Control and Signal Processing}, a senior editor of \emph{IEEE Control Systems Letters}, an associate editor of \emph{IEEE Trans. on Automatic Control}, \emph{IEEE Trans. on Automation Science and Engineering} and \emph{IEEE/ASME Trans. on Mechatronics}. His research interests are adaptive and learning systems with applications in autonomous vehicles and smart energy systems.
\end{IEEEbiography}

\end{document}